\documentclass{article}
    \PassOptionsToPackage{numbers, compress}{natbib}
\usepackage[preprint]{neurips_2026} 
\usepackage{graphicx} 
\usepackage{import}
\usepackage{amsmath} 
\usepackage{amssymb}
\usepackage{amsthm}
\usepackage{xcolor}
\usepackage{natbib}
\usepackage{multirow,makecell}
\usepackage{booktabs}
\usepackage{caption}
\usepackage{subcaption}
\usepackage{wrapfig}
\usepackage{comment}
\usepackage{adjustbox}
\usepackage{hyperref}
\usepackage{float}

\newtheorem{definition}{Definition}

\newtheorem{theorem}{Theorem}
\newtheorem{proposition}{Proposition}

\newcommand{\ie}{\textit{i.e.,}}
\newcommand{\theo}{Theorem} 

\newcommand{\fig}{Fig.}
\newcommand{\tab}{Tab.}
\newcommand{\s}{Sec.}

\newcommand{\unlearning}{MU}

\title{Erased, but Not Gone: \\ Output Forgetting Is Not True Forgetting}
\author{%
  \textbf{Teresa Pui Yee Yong$^{1}$, \quad
  Win Kent Ong$^{1}$, \quad
  Chee Seng Chan$^{1,2}$}\thanks{Corresponding author: \url{cs.chan@um.edu.my}}\\
  $^{1}$ Universiti Malaya, Kuala Lumpur, Malaysia \\
  $^{2}$ VinUniversity, Hanoi, Vietnam\\
}

\date{April 2026}

\begin{document}

\maketitle

\begin{abstract}
Machine unlearning (\unlearning) is commonly judged by \emph{output forgetting}, such as low forget-set accuracy or reduced membership inference. But if output-level success can coexist with retraining-inconsistent residuals in representation space, what kind of forgetting are current evaluations actually certifying? We study this question through \emph{retraining-consistent representation forgetting}, using a population of retrained models (\ie~trained from scratch without the forget data) as an operational reference distribution for ideal forgetting. Across multiple unlearning methods, datasets, and models, our theoretical analysis and empirical results show that standard output-level evaluation can systematically overestimate the success of unlearning. Under this stronger lens, current methods often appear forgotten at the output layer while exhibiting a structured mismatch relative to retraining. They partially align with retraining on forget samples, remain more inconsistent on retain samples, and tend to leave residual discrepancy localized along retraining-related directions rather than diffuse in representation space. This structured mismatch is characterized by forget/retain asymmetry, directional mismatch, and concentrated residuals along retraining-related directions. In many cases, the discrepancies exceed the variability of independent retraining and lead to a risk of information recovery during relearning attacks. Overall, the results suggest that current \unlearning~is often evaluated for \emph{apparent forgetting} rather than retraining-consistent forgetting. More broadly, retraining reveals what output forgetting hides.
\end{abstract}

\section{Introduction}
\label{sec:introduction}

Machine unlearning (\unlearning) is commonly evaluated through \emph{output forgetting}~\cite{kurmanji2023towards,chen2023boundary,tarun2023fast,graves2021amnesiac,foster2024fast}. That is, if a model attains low forget-set accuracy, low output-level membership inference~\cite{carlini2022membership}, and high retain accuracy, it is typically regarded as having forgotten successfully. This has become the default success signal in much of the \unlearning~literature~\cite{chen2023boundary,graves2021amnesiac,foster2024fast,cao2015towards,bourtoule2021machine,xu2024machine,ginart2019making,chundawat2023can,thudi2022unrolling,ong2025maverick}. In this paper, we show through both theoretical analysis and extensive experimental results that this signal is too weak. Across multiple unlearning methods, datasets, and models, methods that appear successful at the output layer often remain inconsistent with retraining in representation space. In other words, current evaluation can systematically mistake \emph{apparent forgetting} for successful forgetting.

To expose this gap, we use \textbf{retraining-consistent representation forgetting} as a stronger evaluative lens. Concretely, we use a population of retrained models, \ie~the model trained from scratch without the forget data, as the operational reference distribution for identifying retraining consistency in representation space. This reference interval is important because comparing an unlearned model to the original reveals how much it has changed, whereas comparing it to retraining distribution reveals whether it has changed in a way consistent with forgetting the designated data, without the restriction to match to a single reference. Under this lens, we find that many existing unlearning methods~\cite{kurmanji2023towards,chen2023boundary,tarun2023fast,graves2021amnesiac,foster2024fast,le2025pour} achieve apparent output-level forgetting while remaining retraining-inconsistent in representation space (see \fig~\ref{fig:teaser}). This shows that output-level success is not sufficient to certify retraining-consistent forgetting. We emphasize that retraining-consistent representation forgetting is not introduced as a universal axiom for all \unlearning~settings. Rather, it serves here as a benchmark-time stronger diagnostic lens for identifying failure modes that weaker endpoint-based criteria can miss.

\begin{wrapfigure}{l}{0.5\linewidth}
\vspace{-13pt}
\centering
\includegraphics[width=\linewidth]{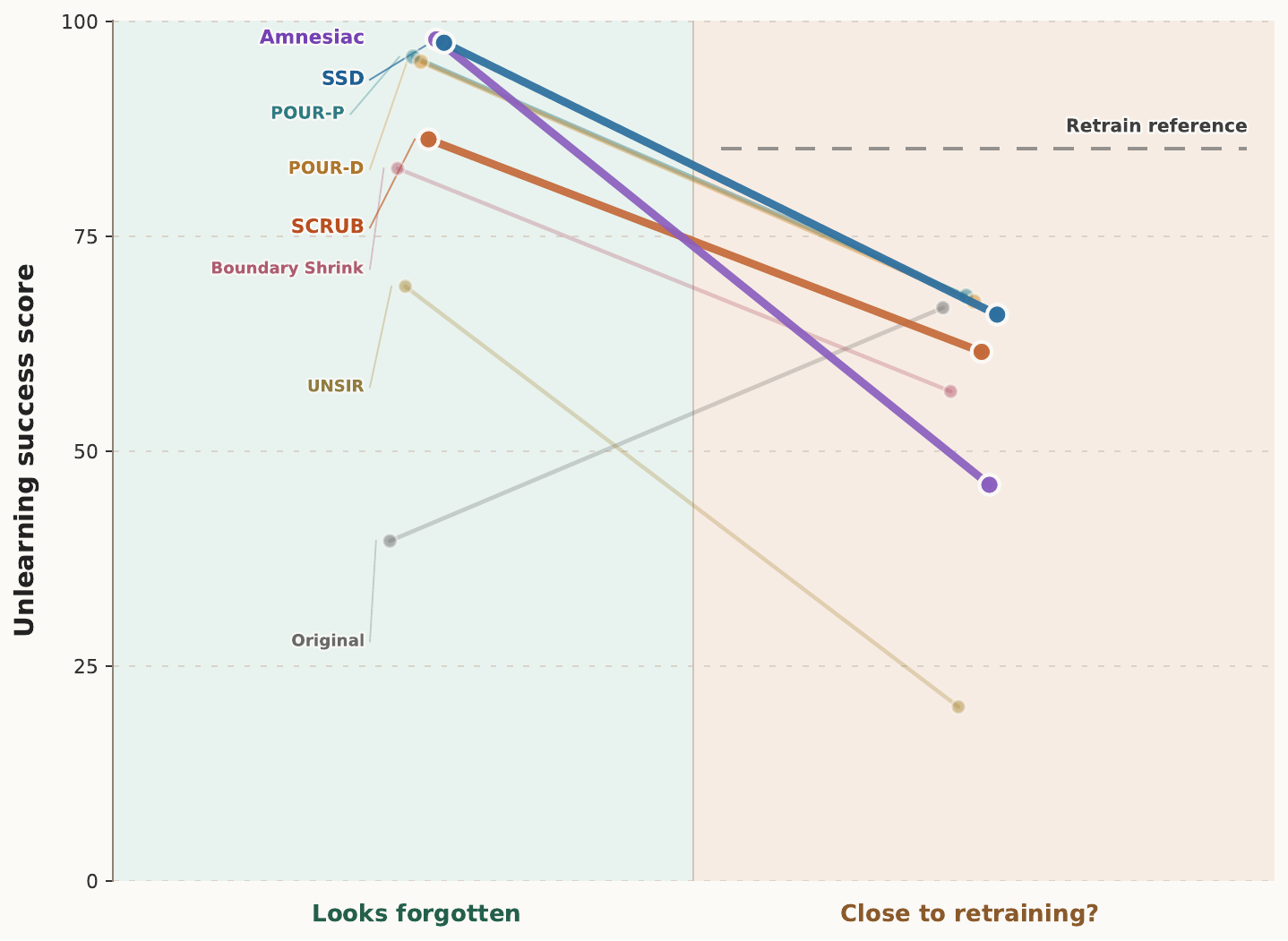}
\caption{\textbf{Looks forgotten, but not close to retraining.}
A schematic summary of the evaluation gap studied in this paper. 
Output-level metrics may suggest successful forgetting, yet the same methods can remain far from the retraining reference in representation space. This discrepancy motivates our retraining-consistent analysis of forget/retain asymmetry, directional mismatch, and structured residuals.}
\label{fig:teaser}
\vspace{-10pt}
\end{wrapfigure}

This raises the central question of the paper, \ie~\emph{\textbf{if output-level success can coexist with retraining-inconsistent residuals in representation space, then what kind of forgetting is current evaluation actually certifying?}}  We further ask whether such residual mismatch follows a stable structure relative to retraining.

As illustrated in \fig~\ref{fig:teaser}, our claim is not simply that representation-level residuals can remain after unlearning.  Rather, under a retraining-consistent lens, these residuals reveal a systematic gap between what current evaluations certify and what exact retraining would produce.  Across the evaluated methods, unlearning often moves in partial agreement with retraining on forget samples, deviates more substantially on retain samples, and leaves residual discrepancy localized along retraining-related directions.  In particular, the mismatch manifests through forget/retain asymmetry, directional mismatch, and structured residuals in representation space. They lie beyond the natural retraining distribution, and result in information recovery during relearning attacks. These findings suggest that many current methods~\cite{foster2024fast,bourtoule2021machine,xu2024machine,ginart2019making,chundawat2023can,thudi2022unrolling,golatkar2020eternal,golatkar2020forgetting} are better understood as \emph{apparent forgetting} at the output level while leaving behind structured representation-level residuals.

We instantiate this analysis in controlled class-unlearning settings, starting with CIFAR-10 with ResNet-18 and extending across dataset complexity, model size, architecture, the forget class, and seed. Across these settings, the same structured mismatch persists beyond a single benchmark configuration. Taken together, these results identify a hidden failure mode in current \unlearning~evaluation and algorithms. The field often treats output-level forgetting as evidence of successful unlearning, even when retraining-consistent forgetting has not been achieved in representation space.

Our contributions are threefold:
\begin{itemize}
    \item We show that \emph{output forgetting} is a misleading success signal for \unlearning, where standard output-level evaluations can systematically overestimate successful forgetting.
    \item We introduce \emph{retraining-consistent representation forgetting} as a stronger evaluative lens, using retrained models population as the operational reference distribution for what ideal forgetting should look like in representation space.
    \item Under this lens, we uncover a \emph{structured failure mode} of current unlearning methods, characterized by forget/retain asymmetry, directional mismatch, localized residual discrepancy. These failure modes exceed the variability of independent retraining, enable faster information recovery during relearning attack, and persist across scales.
\end{itemize}

In this sense, the paper’s central message is simple. Retraining reveals what output forgetting hides.
\section{Problem Formulation and Retraining-Consistent Representation Lens}
\label{sec:preliminaries}
The empirical contradiction motivating this paper is simple. Output-level success does not guarantee consistency with retraining in representation space. The natural next question is therefore not only whether mismatch remains, but how to formalize that mismatch relative to retraining. This section introduces the minimal language needed to pose that question precisely.

\begin{figure}[t]
    \centering
    \includegraphics[width=\linewidth]{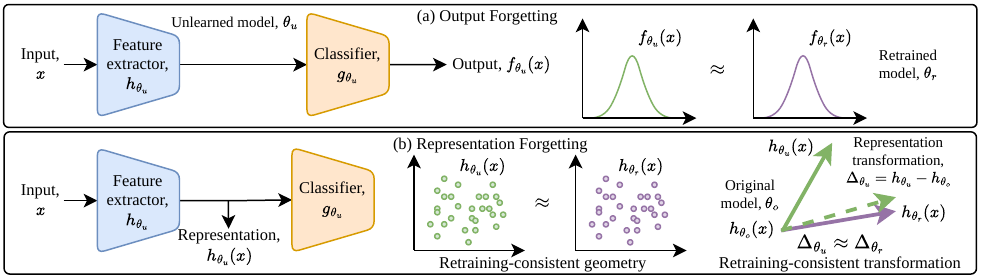}
    \caption{Output forgetting can appear successful while forget set information remains recoverable in representation space. We study this gap through a retraining-consistent representation lens.}
    \label{fig:representation_forgetting}
    \vspace{-10pt}
\end{figure}

\subsection{Preliminaries}
Let $\theta$ denote the initial model parameters, and let $\mathcal{D}=\{(x_i,y_i)\}_{i=1}^N$ be the full training dataset. A learning algorithm $\mathcal{A}$ produces the original model $\theta_o=\mathcal{A}(\theta,\mathcal{D})$. Given a designated forget set $\mathcal{D}_u\subset\mathcal{D}$, the remaining data form the retain set $\mathcal{D}_r=\mathcal{D}\setminus\mathcal{D}_u$. An unlearning algorithm $\mathcal{U}$ produces an unlearned model $\theta_u=\mathcal{U}(\theta_o,\mathcal{D}_u)$ with the goal of removing the influence of $\mathcal{D}_u$ from $\theta_o$. The ideal reference is the distribution of the retrained models $\theta_r \sim \mathcal P_{\theta_r} = \mathbb{P}(\mathcal{A}(\theta,\mathcal{D}_r))$, trained from scratch on the $\mathcal{D}_r$ only, induced by stochastic training.

We write the model as $f_\theta(x)=g_\theta(h_\theta(x))$, where $h_\theta:\mathcal{X}\to\mathbb{R}^d$ denotes the penultimate-layer representation and $g_\theta$ the prediction head. Throughout the paper, $\mathcal P_{\theta_r}$ serves as the \emph{operational reference distribution} for correct forgetting. Comparing $\theta_u$ to $\theta_o$ measures the extent of change, whereas comparing $\theta_u$ to $\mathcal P_{\theta_r}$ assesses whether this change is consistent with retraining without $\mathcal{D}_u$.

\subsection{Two Lenses on Machine Unlearning}

Most~\unlearning~evaluations~\cite{chen2023boundary,graves2021amnesiac,foster2024fast,cao2015towards,bourtoule2021machine,xu2024machine,ginart2019making,chundawat2023can,thudi2022unrolling,ong2025maverick} are still dominated by \emph{output-level success signals}, such as low forget-set accuracy, low output-level membership inference, and high retain accuracy. These metrics assess whether a model \emph{appears} to forget at the prediction layer.

\begin{definition}[Output Forgetting]
\label{def:output_forgetting}
For a set $S$, let $\mathcal{F}_\theta^S=\{f_\theta(x):x\in S\}$ denote the model outputs with empirical distribution $\mathbb{P}_\theta^S(\mathcal{F})$. We say that an unlearned model $\theta_u$ achieves \emph{output forgetting} if its output behavior matches that of the retrained model $\theta_r$ on both the forget set $\mathcal{D}_u$ and the retain set $\mathcal{D}_r$, \ie~ for each $S \in \{ \mathcal{D}_u, \mathcal{D}_r\}$:
\begin{equation}
\mathbb{P}_{\theta_u}^S(\mathcal{F}) \approx \mathbb{P}_{\theta_r}^S(\mathcal{F}), 
\quad \forall S \in \{\mathcal{D}_u, \mathcal{D}_r\}.
\end{equation}
\end{definition}
However, output forgetting alone does not establish whether the model is also consistent with retraining in representation space. To formalize this gap, we use \emph{retraining-consistent representation forgetting} as a stronger evaluative lens, illustrated in \fig~\ref{fig:representation_forgetting}.

For a set $S$, let $\mathcal{H}_\theta^S=\{h_\theta(x):x\in S\}$ denote the representations with empirical distribution $\mathbb{P}_\theta^S(\mathcal{H})$. We define the representation transformation from $\theta$ to $\theta'$ on $S$ as $\mathcal{T}_{\theta\rightarrow\theta'}^S = \{h_{\theta'}(x)-h_\theta(x):x\in S\}$, with empirical distribution $\mathbb{P}_{\theta\rightarrow\theta'}^S(\mathcal{T})$.

\begin{definition}[Retraining-Consistent Representation Forgetting]
\label{def:rep_forgetting}
We define \emph{retraining-consistent representation forgetting} as follows. Under this lens, an unlearned model $\theta_u$ is assessed not only by whether its representations resemble the representations induced by the population of retrained models $\theta_r\sim\mathcal P_{\theta_r}$, and also by whether its representation changes relative to the original model $\theta_o$ 
match those induced by retraining. Concretely, for each 
$S \in \{\mathcal{D}_u, \mathcal{D}_r\}$, we assess:
\begin{equation} \label{eq:rep_and_shift_match}
\begin{aligned}
\mathbb{P}_{\theta_u}^S(\mathcal{H}) 
&\approx \mathbb{P}_{\theta_r}^S(\mathcal{H}), \\
\mathbb{P}_{\theta_o\rightarrow\theta_u}^S(\mathcal{T}) 
&\approx \mathbb{P}_{\theta_o\rightarrow\theta_r}^S(\mathcal{T}),
\end{aligned}
\quad \forall S \in \{\mathcal{D}_u, \mathcal{D}_r\}, \quad \theta_r \sim \mathcal P_{\theta_r}.
\end{equation}
\end{definition}
We use this notion as a \emph{stronger evaluative lens}, not as a universal axiom for all \unlearning~settings. Its role in this paper is to diagnose the residual mismatch in representation space hidden by similarity at the output layer. In our benchmark-scale setting, retraining provides the ideal reference to test the removal of $D_u$'s additional influence~\cite{triantafillou2026your}, and it is used as an empirical reference population instead of a single reference which unlearning must match exactly, allowing for optimization variability.

\subsection{Why Output Forgetting is Too Weak}
\label{sec:output-too-weak}
The following result formalizes a limitation of output-level metrics. See Appen.~\ref{appendix:proof_main_theorem} for the full proof.

\begin{theorem}[Output Forgetting Does Not Imply Retraining-Consistent Representation Forgetting]
\label{theo:forget_discrepancy}
Let $f_\theta(x)=g_\theta(h_\theta(x))$ be a classifier composed of a feature extractor
$h_\theta:\mathcal{X}\to\mathbb{R}^d$ and a prediction head $g_\theta$.
There exists an unlearned model $\theta_u$ such that
$f_{\theta_u}(x)=f_{\theta_r}(x), \quad \forall x\in\mathcal{D},$
while
$\mathbb{P}_{\theta_u}^{S}(\mathcal{H})
\neq
\mathbb{P}_{\theta_r}^{S}(\mathcal{H})$
for at least one $S\in\{\mathcal{D}_u,\mathcal{D}_r\}$.
Hence, perfect output alignment with the retrained model does not guarantee retraining-consistent representations.
\end{theorem}

\subsection{Diagnostics for the Hidden Failure Mode}
\label{sec:diagnostics-def}
Our protocol is designed as a benchmark-time diagnostic toolkit for exposing \emph{how} current methods fail under a retraining-consistent representation lens.

\paragraph{Representation similarity.}
To quantify how closely $\theta_u$ matches $\theta_r$ in representation geometry, we use centered kernel alignment (CKA)~\cite{kornblith2019similarity}. Given representation matrices $X,Y\in\mathbb{R}^{n\times d}$ evaluated on the same data, linear CKA is defined as:
\begin{equation}
\mathrm{CKA}(X,Y)=
\frac{\langle XX^\top,YY^\top\rangle_F}{\|XX^\top\|_F\|YY^\top\|_F}.
\end{equation}
Higher CKA indicates that the unlearned representation is closer to the retrained representation. See Appendix~\ref{appendix:rep_metric_layer} for the analysis on representation layers sensitivity and similarity measure robustness.

\paragraph{Transformation mismatch.}
To characterize how unlearning moves representations relative to retraining, we define the mean representation shift from $\theta_o$ to $\theta'$ on a set $S$ as:
\begin{equation}
\Delta_{\theta_o\rightarrow\theta'}^S
=
\frac{1}{|S|}\sum_{x\in S} h_{\theta'}(x)
-
\frac{1}{|S|}\sum_{x\in S} h_{\theta_o}(x),
\qquad \theta'\in\{\theta_u,\theta_r\}.
\end{equation}
We then measure directional alignment via cosine similarity:
\begin{equation}
\cos\!\left(\Delta_{\theta_o\rightarrow\theta_u}^S,\Delta_{\theta_o\rightarrow\theta_r}^S\right)
=
\frac{\Delta_{\theta_o\rightarrow\theta_u}^S\cdot\Delta_{\theta_o\rightarrow\theta_r}^S}
{\|\Delta_{\theta_o\rightarrow\theta_u}^S\|\|\Delta_{\theta_o\rightarrow\theta_r}^S\|}.
\end{equation}
This quantity diagnoses whether unlearning follows a retraining-consistent transformation on forget and retain samples. To further localize the discrepancy, we decompose representations along the retraining direction. Let
$v^S=\frac{\Delta_{\theta_o\rightarrow\theta_r}^S}{\|\Delta_{\theta_o\rightarrow\theta_r}^S\|},$
we decompose
$h_{\theta'}^{\parallel}(x)=(h_{\theta'}(x)\cdot v^S)\,v^S, \, h_{\theta'}^{\perp}(x)=h_{\theta'}(x)-h_{\theta'}^{\parallel}(x)$, and compute the parallel CKA using $h_{\theta'}^{\parallel}(x)$ and orthogonal CKA using $h_{\theta'}^{\perp}(x)$.

\paragraph{Retraining reference calibration.}
To calibrate the diagnostic against $\mathcal{P}_{\theta_r}$, we train 10 retrained models using different random seeds. We then perform model-level bootstrapping by resampling model identities 10,000 times, constructing empirical reference intervals from all retrain pairs, and recomputing the outside rate for each model. Under the three-way decision rule, an outside rate of at least 95$\%$ is classified as Detected, at most 5$\%$ as No signature, and intermediate values as Inconclusive. Our inference is deliberately one-sided: falling outside the reference interval indicates detectable retraining inconsistency under the diagnostic, whereas falling inside indicates no signature is detected. This does not certify complete forgetting or as a sufficient privacy guarantee.

\paragraph{Relearning attack.}
As a diagnostic of the operational implications of representation mismatch, we conduct a relearning attack on models that achieve output forgetting by fine-tuning them on a small subset of $D_u$. We measure the number of epochs required to recover forget-test accuracy to the 90$\%$ threshold. This evaluates whether the generalization of the unlearned model remains recoverable under an adversary’s attempt to reconstruct knowledge of the forget class.
\section{Experimental Results}
\label{sec:results}

\subsection{Experimental Setup}
\paragraph{Datasets, models, and unlearning scenario.}
We study class unlearning, where the goal is to remove the influence of one designated class while preserving performance on the retain set $\mathcal{D}_r$. We evaluate on CIFAR-10, CIFAR-100~\cite{krizhevsky2009learning}, and TinyImageNet~\cite{le2015tiny}, which are standard image classification benchmarks for class-level unlearning. Our primary analysis is conducted on ResNet-18~\cite{kodge2024deep}, with ResNet-50 for model-capacity scaling and ViT-Tiny~\cite{dosovitskiy2020image} for cross-backbone validation.

\paragraph{Baselines and unlearning methods.}
We use the original model $\theta_o$ (without unlearning) as the lower reference, a population of retrained models $\{\theta_r^{(i)}\}_{i=1}^{10}$ (trained from scratch on $\mathcal{D}_r$ only) to operationalize $\mathcal P_{\theta_r}$, and a held-out retrained model $\theta_r$ as a positive control. We evaluate five output-oriented methods, including SCRUB~\cite{kurmanji2023towards}, Boundary Shrink~\cite{chen2023boundary}, UNSIR~\cite{tarun2023fast}, Amnesiac~\cite{graves2021amnesiac}, and SSD~\cite{foster2024fast}. We also include POUR-P and POUR-D~\cite{le2025pour}, which are representation-aware unlearning methods. Additional experimental details are provided in Appendix~\ref{appendix:experimental_details}.

\paragraph{Evaluation metrics.}
Following \s~\ref{sec:preliminaries}, we report both output-level and representation-level quantities. Output-level metrics are forget-set accuracy $\text{Acc}_u$, retain-set accuracy $\text{Acc}_r$, and membership inference $\text{MIA}$. Representation-level diagnostics are the mean forget-set and retain-set CKA over $\theta_r^{(i)}$ ($\text{CKA}_u$ and $\text{CKA}_r$), the mean directional alignment between unlearning and retraining shifts over $\theta_r^{(i)}$ on forget-set and retain-set ($\cos(\Delta_u)$ and $\cos(\Delta_r)$), and subspace-localized parallel and orthogonal $\text{CKA}_u$. Lower $\text{Acc}_u$ and lower MIA indicate stronger forgetting, whereas higher $\text{Acc}_r$, higher CKA, and higher directional alignment indicate closer approximation to retraining.

\begin{table}[t]
\centering
\begin{adjustbox}{max width=\linewidth}
\begin{tabular}{l|ccc|cc|l}
\toprule
\multirow{2}{*}{\textbf{Method}} 
& \multicolumn{3}{c|}{\makecell[c]{\textbf{Appears forgotten under} \\ \textbf{standard output evaluation}}}
& \multicolumn{2}{c|}{\makecell[c]{\textbf{Actually close to retraining} \\ \textbf{in representation space?}}}
& \multirow{2}{*}{\makecell[c]{\textbf{Interpretation}}} \\
\cmidrule(lr){2-4}\cmidrule(lr){5-6}
& \makecell[c]{\textbf{Acc$_u$} \\ \textbf{(\%) $\downarrow$}}
& \makecell[c]{\textbf{Acc$_r$} \\ \textbf{(\%) $\uparrow$}}
& \makecell[c]{\textbf{MIA} \\ \textbf{(\%) $\downarrow$}}
& \makecell[c]{\textbf{CKA$_u$} \\ \textbf{$\uparrow$}}
& \makecell[c]{\textbf{CKA$_r$} \\ \textbf{$\uparrow$}}
& \\
\midrule
\textbf{Original}        
& 98.65 & 98.35 & 81.1 
& 0.580 & 0.916
& \makecell[l]{Original model} \\
\textbf{Held-out retrain}         
& 0.00 & 98.74 & 26.3 
& 0.800 & 0.926
& \makecell[l]{Reference for ideal forgetting} \\
\midrule
\textbf{SCRUB}~\cite{kurmanji2023towards}          
& \textbf{0.00} & \textbf{98.98} & 40.2 
& 0.517 & 0.819
& \makecell[l]{Forgotten in outputs, far from retrained} \\
\textbf{Boundary Shrink}~\cite{chen2023boundary} 
& 13.55 & 84.27 & 22.1 
& 0.443 & 0.750
& \makecell[l]{Weak on both forgetting lens} \\
\textbf{UNSIR}~\cite{tarun2023fast}           
& \textbf{0.00} & 27.55 & 20.1 
& 0.017 & 0.044
& \makecell[l]{Damages the model while forgetting} \\
\textbf{Amnesiac}~\cite{graves2021amnesiac}        
& \textbf{0.00} & \textbf{97.87} & \textbf{4.2} 
& 0.087 & 0.786
& \makecell[l]{Strong output forgetting, inconsistent representation} \\
\textbf{SSD}~\cite{foster2024fast}             
& \textbf{0.00} & \textbf{97.39} & \textbf{4.9} 
& 0.603 & 0.897
& \makecell[l]{Strong output forgetting, inconsistent representation} \\
\textbf{POUR-P}~\cite{le2025pour}          
& \textbf{0.00} & \textbf{98.36} & 10.8 
& 0.627 & 0.915
& \makecell[l]{More consistent representations, short of retraining} \\
\textbf{POUR-D}~\cite{le2025pour}          
& 0.02 & 96.90 & 11.0 
& 0.630 & 0.891
& \makecell[l]{More consistent representations, short of retraining} \\
\bottomrule
\end{tabular}
\end{adjustbox}
\caption{Standard output-level evaluation can overestimate successful unlearning. Several methods, notably SCRUB\cite{kurmanji2023towards}, Amnesiac\cite{graves2021amnesiac}, and SSD\cite{foster2024fast}, appear strong under forget accuracy, retain accuracy, and MIA, yet remain substantially inconsistent with the retrained reference in representation space. This gap is the first empirical signal of the hidden failure mode studied in this paper.}
\label{tab:main_results}
\vspace{-8pt}
\end{table}

\subsection{Standard Output-level Evaluation overestimates Successful Unlearning}

We begin with the paper’s central claim, \ie~whether the field’s default success signal, namely output-level forgetting, can overestimate successful unlearning. The answer is \textbf{YES}.

As shown in \tab~\ref{tab:main_results}, several methods, including SCRUB~\cite{kurmanji2023towards}, Amnesiac~\cite{graves2021amnesiac}, and SSD~\cite{foster2024fast}, achieve near-zero $\text{Acc}_u$ together with reduced $\text{MIA}$, which would typically be interpreted as successful forgetting under standard output-level evaluation. However, these same models remain far from the retrained reference distribution in representation space. For instance, under training variation the mean $\text{CKA}_u$ between retained models is 0.800, but the unlearned models are substantially lower than that. The calibration against retraining interval in \s~\ref{sec: calibration} further validates this discrepancy. 

Under the stronger retraining-consistent representation lens, these methods remain substantially inconsistent, indicating that standard output-level evaluation paints a more optimistic picture than the stronger representation-level diagnostics support. This is the first component of the hidden failure mode - \emph{the field’s dominant success signal is too weak}. Linear probing and t-SNE visualization further support this claim in Appendix~\ref{appendix:linear_probe}.

\begin{figure}[t]
    \centering
    \begin{minipage}{0.54\textwidth}
        \renewcommand{\arraystretch}{1.2}
        \centering
        {\small
        \begin{adjustbox}{max width=\linewidth}
        \begin{tabular}{l|ccc}
            \toprule
            \textbf{Method}  & \makecell[c]{\textbf{Forget set} \\ \(\mathbf{\cos(\Delta_u)} \uparrow\)}
            & \makecell[c]{\textbf{Retain set} \\ \(\mathbf{\cos(\Delta_r)} \uparrow\)}
            & \makecell[c]{\textbf{Asymmetry gap} \\ \(\mathbf{\cos(\Delta_u)-\cos(\Delta_r)} \downarrow\)} \\
            
            \midrule
            Held-out retrain & 0.956 & 0.543 & 0.413 \\
            SCRUB~\cite{kurmanji2023towards} & 0.925 & \textbf{$-$0.280} & \textbf{1.205} \\
            Boundary Shrink~\cite{chen2023boundary} & 0.884 & \textbf{$-$0.281} & \textbf{1.165} \\
            UNSIR~\cite{tarun2023fast} & 0.703 & 0.144 & 0.559 \\
            Amnesiac~\cite{graves2021amnesiac} & 0.767 & 0.278 & 0.489 \\
            SSD~\cite{foster2024fast} & 0.884 & \textbf{$-$0.336} & \textbf{1.220} \\
            POUR-P~\cite{le2025pour} & 0.816 & 0.401 & 0.415 \\
            POUR-D~\cite{le2025pour} & 0.849 & 0.064 & 0.785 \\
            \bottomrule
        \end{tabular}
        \end{adjustbox}
        }
        \captionsetup{type=table}
        \caption{The mismatch with retraining is asymmetric across forget and retain samples. Most methods align more strongly with retraining on the forget set than on the retain set; negative \(\cos(\Delta_r)\) indicate especially poor retain-side alignment, and the asymmetry gap \(\cos(\Delta_u)-\cos(\Delta_r)\) makes this imbalance explicit.}
        \label{tab:directional_alignment}
    \end{minipage}
    \hfill
    \nextfloat
    \begin{minipage}{0.43\textwidth}
        \centering
        \includegraphics[width=\textwidth]{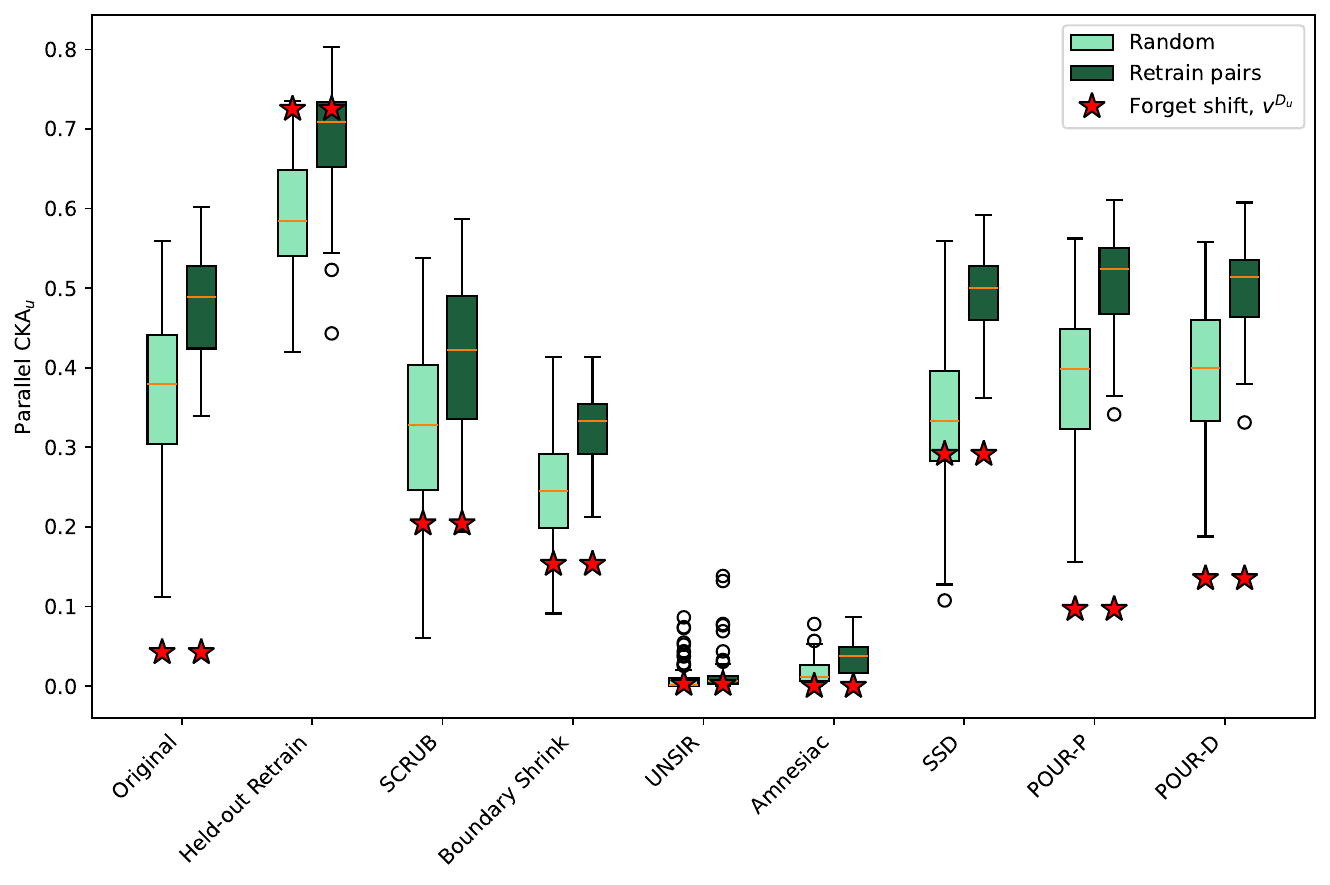}
        \caption{Residual discrepancy is structured rather than random. Several methods exhibit stronger mismatch along $v^{D_u}$ while some have weaker tendency.}
        \label{fig:parallel_cka}
    \end{minipage}
    \vspace{-5pt}
\end{figure}

\subsection{Mismatch with Retraining is Asymmetric across Forget and Retain Samples}

We next ask whether the mismatch with retraining is uniform across samples. It is \textbf{NOT}. 

Instead, the discrepancy is asymmetric between the forget set $\mathcal{D}_u$ and the retain set $\mathcal{D}_r$. As shown in \tab~\ref{tab:directional_alignment}, unlearning methods often exhibit relatively high directional alignment with retraining on $\mathcal{D}_u$, but much weaker alignment on $\mathcal{D}_r$. At the same time, this directional agreement on $\mathcal{D}_u$ does not translate into retraining-consistent representation as $\text{CKA}_u$ remains substantially lower than the retrained reference. In contrast, $\mathcal{D}_r$ often preserves higher $\text{CKA}_r$ despite weaker directional alignment. Shift magnitude diagnostics in Appendix~\ref{appendix:shift_mag_ratio} complement the results descriptively.

This asymmetry is important. It shows that current methods are not simply “wrong everywhere”, but exhibit a partial and uneven approximation of retraining. They often move in the right coarse direction on forget samples, yet fail to reproduce retraining-consistent behavior on retain samples and forget set representation geometry. This is the second component of the hidden failure mode - \emph{unlearning is asymmetrically misaligned with retraining across forget and retain samples}.

\subsection{Residual Discrepancy is Structured rather than Random}

The next question is whether the remaining discrepancy is merely diffuse error, or whether it has structure. Our results support the latter. 

\fig~\ref{fig:parallel_cka} shows that several unlearned models exhibit substantially lower $\text{CKA}_u$ along $v^{D_u}$ (defined in Sec.~\ref{sec:diagnostics-def}) than along orthogonal components (figure in Appendix~\ref{appendix:subspace_analysis}). The comparisons with random directions and retrained pairs shift directions indicate that this effect cannot be fully explained by matched-dimensional random projection or stochastic retraining variation in particular. Other methods exhibit a weaker tendency but toward the same directional mismatch. 

The discrepancy is hence not well described as diffuse noise around retraining. Instead, it is preferentially organized along retraining-related directions. This is the third component of the hidden failure mode - current methods often leave behind a \emph{structured residual} rather than diffused. Geometry is useful here to reveal \emph{how} the hidden failure mode is organised. Additional interpretations on the construction control and subspace presentations are provided in Appendix~\ref{appendix:subspace_analysis}.

\subsection{Structured Mismatch Exceeds Retraining Variability}
\label{sec: calibration}
We then ask whether the structured discrepancy is significant empirically. We find that most unlearning methods fall outside the empirical variability observed among valid retrained models in the representation diagnosis on CIFAR-10 with ResNet-18. 

In \tab~\ref{tab:signature}, the retraining interval captures the natural variation between valid retrain pairs across the six diagnostics. The outside-rate and decision rule are described in \s~\ref{sec:diagnostics-def}. The positive control held-out retrained model supports this calibration: it shows no detected signature across the CKA and directional diagnostics, while its $\cos(\Delta_u)$ is inconclusive. In contrast, every unlearning method is detected by all diagnostics, except POUR-P remains inconclusive on the two retain diagnostics. 

These results indicate that the measured discrepancies extend beyond the empirical retraining variability captured by the reference population. Importantly, the conclusion is one-sided: falling outside the reference interval indicates retraining inconsistency under the measured diagnostic, whereas falling inside indicates that no signature is detected. Thus, the calibrated audit reveals retraining inconsistency that is not apparent from near-zero forget accuracy alone.

\begin{table}[t]
    \renewcommand{\arraystretch}{1.2}
    \centering
    {\small
    \begin{adjustbox}{max width=\linewidth}
    \begin{tabular}{l|cccccc}
        \toprule
        \multirow{2}{*}{\textbf{Model}}
        & \textbf{CKA$_u$} & \textbf{CKA$_r$} & \textbf{$\cos(\Delta_u)$} & \textbf{$\cos(\Delta_r)$}
        & \textbf{Parallel CKA$_u$} & \textbf{Orthogonal CKA$_u$} \\
        \cmidrule(lr){2-7}
        & \multicolumn{6}{c}{\textbf{Retraining interval}} \\
        \midrule
        Retrains population & [0.7772, 0.8360] & [0.9144, 0.9357] & [0.9342, 0.9598] & [0.4385, 0.6536] & [0.6538, 0.7925] & [0.7666, 0.8350] \\
        \midrule
        \textbf{Method} & \multicolumn{6}{c}{\textbf{Verdict (Outside-rate $\%$)}} \\
        \midrule
        Held-out retrain & No signature (0.09) & No signature (0.1) & Inconclusive (17.96) & No signature (1.97) & No signature (0.03) & No signature (0.08) \\
        SCRUB & Detected (100) & Detected (100) & Detected (99.19) & Detected (100) & Detected (100) & Detected (100) \\
        Boundary Shrink & Detected (100) & Detected (100) & Detected (100) & Detected (100) & Detected (100) & Detected (100) \\
        UNSIR & Detected (100) & Detected (100) & Detected (100) & Detected (100) & Detected (100) & Detected (100) \\
        Amnesiac & Detected (100) & Detected (100) & Detected (100) & Detected (100) & Detected (100) & Detected (100) \\
        SSD & Detected (100) & Detected (100) & Detected (100) & Detected (100) & Detected (100) & Detected (100) \\
        POUR-P & Detected (100) & Inconclusive (67.37) & Detected (100) & Inconclusive (93.99) & Detected (100) & Detected (100) \\
        POUR-D & Detected (100) & Detected (100) & Detected (100) & Detected (100) & Detected (100) & Detected (100) \\
        \bottomrule
    \end{tabular}
    \end{adjustbox}
    }
    \caption{Calibrated representation audit on CIFAR-10 with ResNet-18. The held-out retrained model acts as a positive control, whereas every unlearning method is detected on all six diagnostics, except POUR-P, which is inconclusive on two retain-side diagnostics. The test is one-sided: falling outside the interval indicates retraining inconsistency, while falling inside means no signature is detected. The audit exposes retraining inconsistencies beyond variability of independent retraining.}
    \label{tab:signature}
    \vspace{-10pt}
\end{table}

\begin{figure}[t]
    \centering
    \begin{minipage}{0.63\textwidth}
        \renewcommand{\arraystretch}{1.2}
        \centering
        {\small
        \begin{adjustbox}{max width=\linewidth}
        \begin{tabular}{l|ccc}
            \toprule
            \textbf{Method}
            & \makecell[c]{\textbf{Recovery} \\ \textbf{success}}
            & \makecell[c]{\textbf{Recovery} \\ \textbf{epoch}}
            & \makecell[c]{\textbf{Retain test accuracy} \\ \textbf{at recovery}} \\
            \midrule
            Retrain & 7/10 models & 44--48 & 35.89--57.71\% \\
            SSD & 3/3 runs & 25--35 & 79.07--82.63\% \\
            POUR-D & 2/3 runs & 15--16 & 76.79--78.01\% \\
            SCRUB & 3/3 runs & 4--5 & 0\% \\
            Amnesiac & 0/3 runs & Not reached & N/A (no recovery) \\
            \bottomrule
        \end{tabular}
        \end{adjustbox}
        }
        \captionsetup{type=table}
        \caption{Forget information recovers faster than in valid retrains. The ranges summarize successful runs across seeds. SSD and POUR-D with successful output forgetting recover earlier than retrains with higher retain test accuracy.}
        \label{tab:recovery}
    \end{minipage}
    \hfill
    \begin{minipage}{0.32\textwidth}
        \centering
        \includegraphics[width=\textwidth]{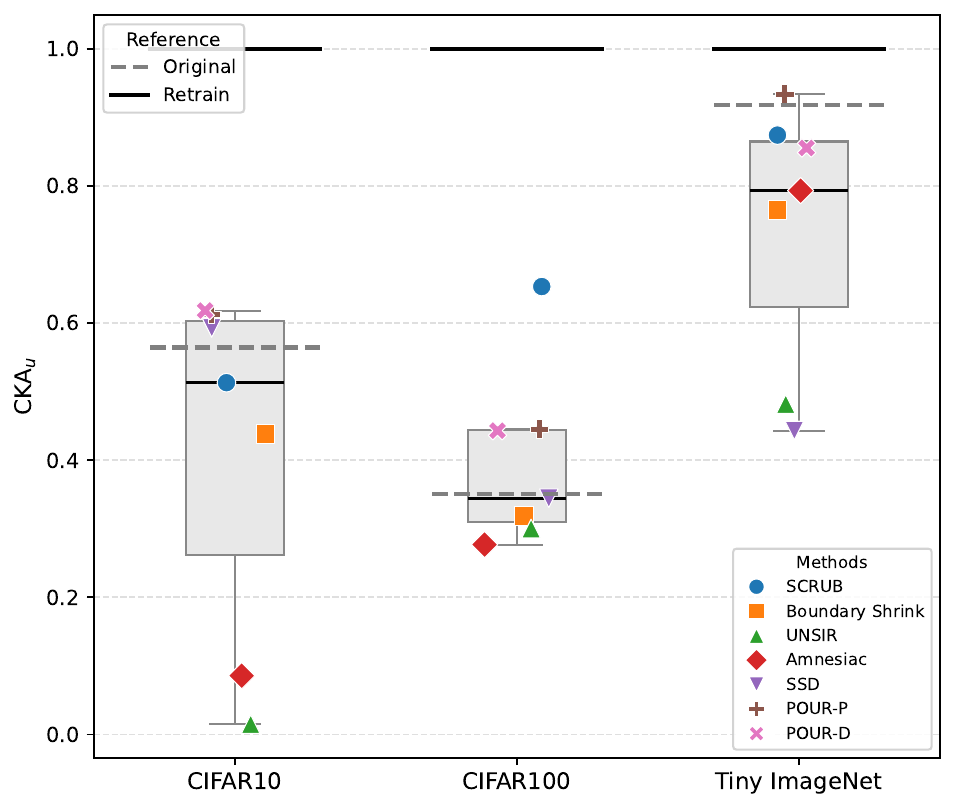}
        \caption{Representation similarity across dataset.}
        \label{fig:dataset_scaling}
    \end{minipage}
    \vspace{-5pt}
\end{figure}

\subsection{Faster Recovery under Relearning Attack}
\label{sec: relearning}
Next, we ask whether the discrepancy leads to practical implications. With re-exposure to five forget samples, the residuals produce a measurable information recovery effect beyond the output metrics.

From \tab~\ref{tab:recovery}, among methods that achieve output forgetting, SSD and POUR-D reach the $\geq90\%$ forget-test accuracy threshold earlier than successful retrains, with non-overlapping recovery-epoch ranges, while also achieving higher endpoint accuracy. SCRUB recovers within 4--5 epochs but with 0$\%$ test $\text{Acc}_r$, illustrating why recovery speed should be considered jointly with retained performance. Amnesiac provides a complementary non-recovery result within the evaluated budget. 

Overall, for methods such as SSD and POUR-D that exhibit successful output forgetting, representation inconsistency is accompanied by earlier threshold recovery with retained performance under controlled re-exposure. This provides a measurable operational consequence under matched re-exposure, complementary to deployment-time privacy and security evaluation.

\begin{figure}[t]
    \centering
    \begin{minipage}{0.24\textwidth}
        \centering
        \includegraphics[width=\textwidth]{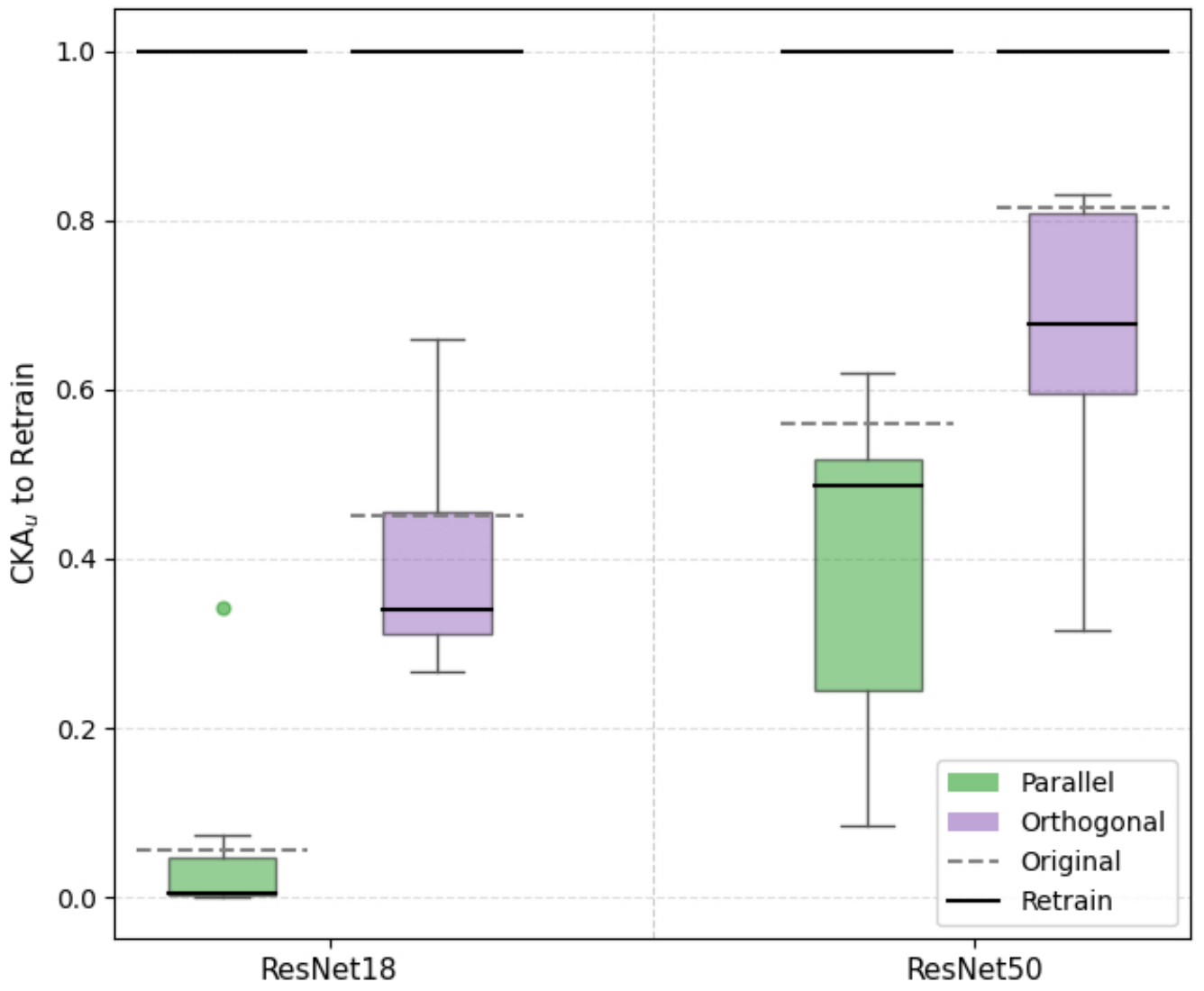}
        \caption{Localized residual discrepancy across model size.}
        \label{fig:model_scaling}
    \end{minipage}
    \hfill
    \nextfloat
    \begin{minipage}{0.24\textwidth}
        \centering
        \includegraphics[width=\textwidth]{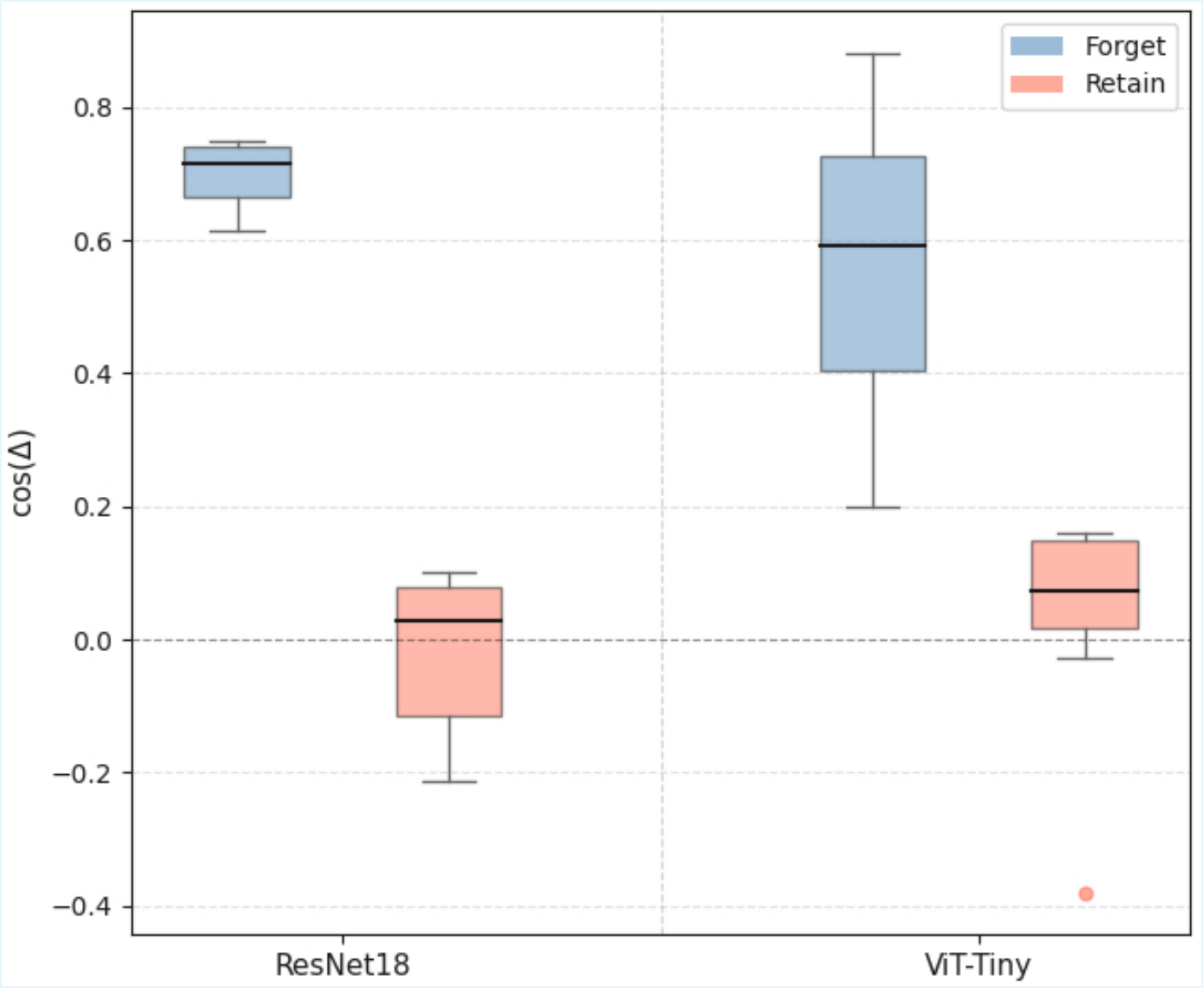}
        \caption{Directional asymmetry across backbones.}
        \label{fig:cross_backbone}
    \end{minipage}
    \hfill
    \nextfloat
    \begin{minipage}{0.24\textwidth}
        \centering
        \includegraphics[width=\textwidth]{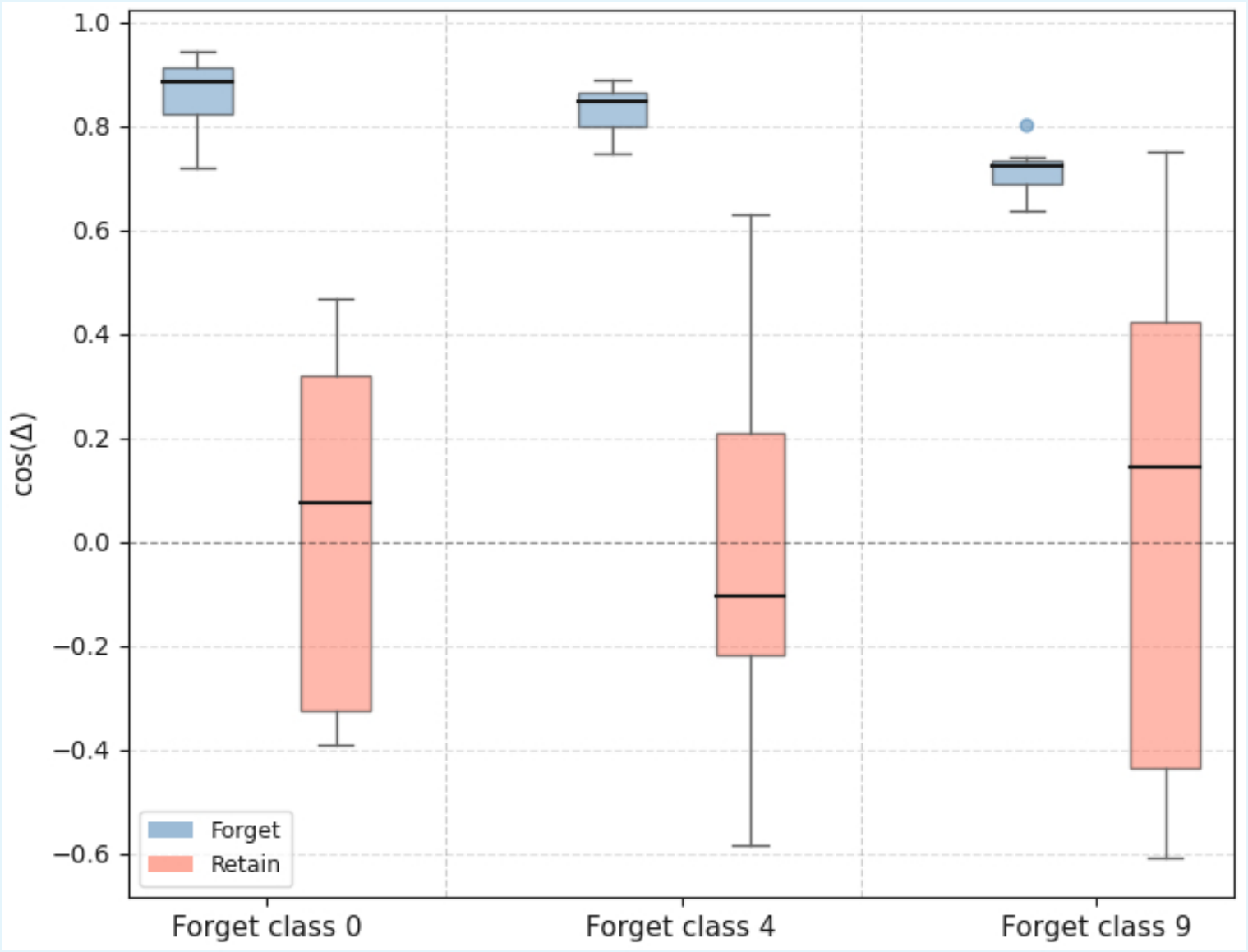}
        \caption{Directional asymmetry across forget classes.}
        \label{fig:multi_forget_class}
    \end{minipage}
    \hfill
    \nextfloat
    \begin{minipage}{0.24\textwidth}
        \centering
        \includegraphics[width=\textwidth]{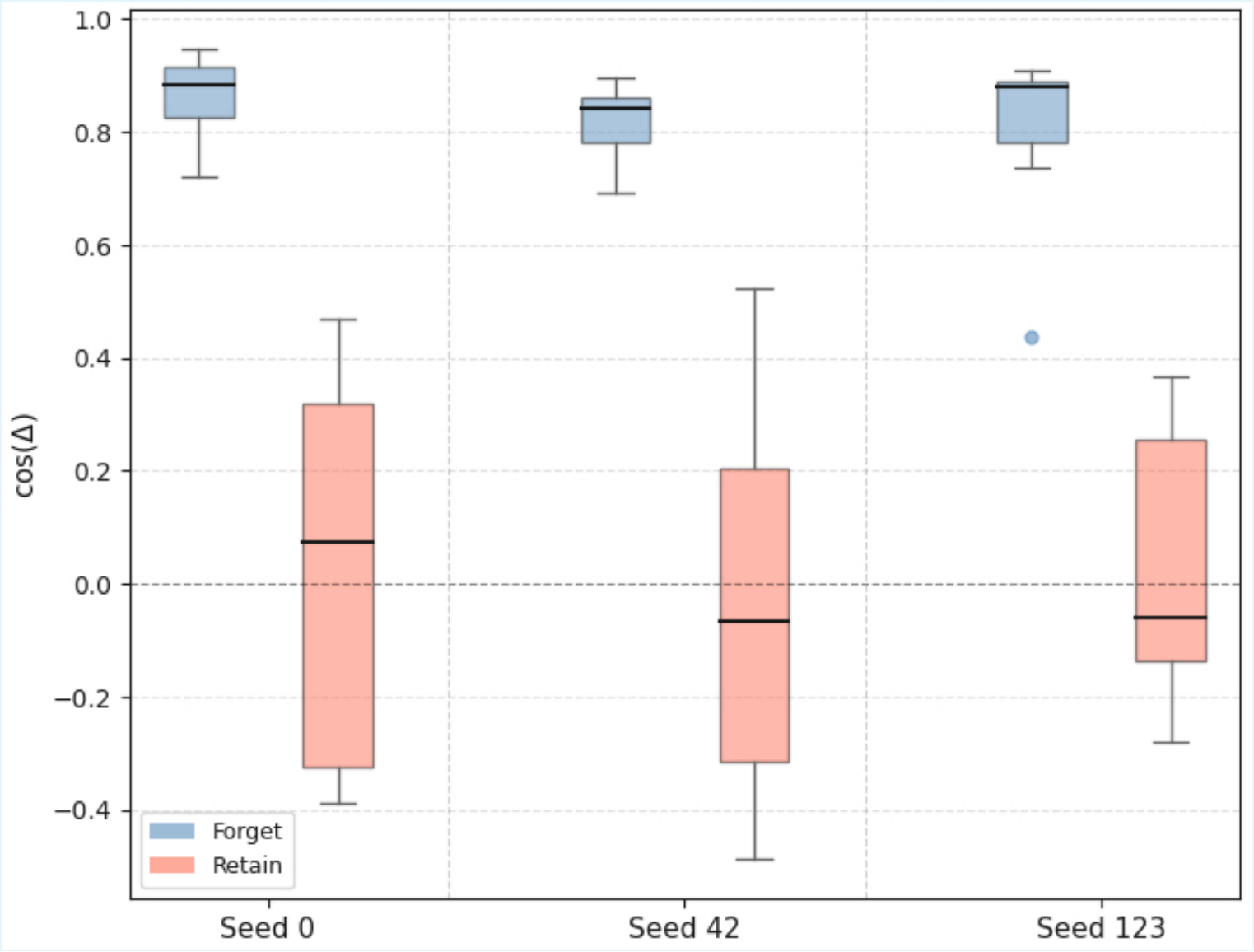}
        \caption{Directional asymmetry across random seeds.}
        \label{fig:seed_robustness}
    \end{minipage}
    \vspace{-10pt}
\end{figure}

\subsection{Same Structured Mismatch Persists across Scale}
\label{sec: scaling}
We test whether the diagnosed mismatch is a narrow artifact of one benchmark setting or a stable property of current unlearning behavior. Specifically, we rule out four weaker explanations, \ie~easy datasets, underpowered models, convolutional backbones, and favorable class or seed choices. The diagnostics use one retrained reference and serve as qualitative analysis instead of calibrated tests.

\paragraph{Dataset complexity.}
A first weak explanation is that the diagnosed contradiction is merely a byproduct of an easy benchmark. To test this, we increase dataset complexity to CIFAR-100 and TinyImageNet. Empirically, the same pattern remains. The $\text{Acc}_u$ stays near zero while $\text{CKA}_u$ remains substantially below the retrained reference (\fig~\ref{fig:dataset_scaling}). This shows that the paper’s central contradiction persists as task complexity increases, where models can continue to look forgotten at the output layer while remaining retraining-inconsistent in the representation space.

\paragraph{Model size.}
A second weak explanation is that the failure mode is simply a small-model artifact. We increase model capacity to ResNet-50 on CIFAR-100. Although overall representation quality improves, the mismatch remains, and the discrepancy becomes more localized along retraining-related directions (\fig~\ref{fig:model_scaling}). The failure mode therefore survives scaling with a sharper structure.

\paragraph{Architecture.}
A third weak explanation is that the diagnosed mismatch is tied to one specific representation mechanism. We then replace the backbone with ViT-Tiny because convolutional backbones and vision transformers differ in their inductive biases, feature formation, and representation geometries. Empirically, the same direction/geometry gap persists (\fig~\ref{fig:cross_backbone}), showing that the failure mode is not merely a convolution-specific phenomenon, but survives a change in architecture.

\paragraph{Forget class and seed.}
A fourth weak explanation is that the observed pattern is contingent on one favorable class choice or one lucky training run. Thus, we vary the forgotten class and the random seed. The same forget/retain asymmetry recurs across classes and seeds with low variance (\fig~\ref{fig:multi_forget_class} and ~\ref{fig:seed_robustness}), indicating that the diagnosed mismatch is neither class-specific nor a stochastic accident.

Taken together, these experiments show that the observed gap is not a narrow artifact of dataset, model, architecture, class, or seed. Rather, the same structured mismatch relative to retraining repeatedly emerges: output-level success remains more optimistic than representation-level diagnostics, while forget/retain asymmetry, directional inconsistency, and localized discrepancy persist. This strengthens the claim that the hidden failure mode identified reflects a broader weakness of current \unlearning~evaluation and algorithms. Extended scaling results are provided in Appendix~\ref{appendix:extended_scaling}.

\section{Related Work}
\label{sec:related_work}

\paragraph{Existing success signals for machine unlearning.}
Most \unlearning~methods are designed through \emph{output-level forgetting}, typically by manipulating decision boundaries, output distributions, or forget-set predictions~\cite{chen2023boundary,graves2021amnesiac,foster2024fast,cao2015towards,bourtoule2021machine,xu2024machine,ginart2019making,chundawat2023can,thudi2022unrolling}. As a result, the dominant evaluation framework focuses on output behavior, such as forget accuracy, retain accuracy and logit-level MIA~\cite{chen2023boundary,tarun2023fast,foster2024fast,chundawat2023can,shokri2017membership}. However, they do not establish whether forgot information has been removed from its representation. Our work shows that output-level metrics can be systematically too weak as success signals.

\paragraph{Representation-level unlearning and evaluation.}
A growing body of work instead operates directly in the representation space with strategies include projection, contrastive separation, adversarial erasure, and feature-space transformation~\cite{le2025pour,kodge2024deep,lee2025esc,hoang2024learn,almudevar2026representation,zhang2024contrastive,lee2026erase,cotogni2023duck,sepahvand2025selective}. In parallel, recent evaluation protocols likewise incorporate representation diagnostics like CKA, representation-level MIA, linear probing, and mutual-information measures~\cite{le2025pour,lee2026erase,sepahvand2025selective,guo2022efficient,kim2026we,jeon2026information,gao2026illusion}. These works show that unlearned models can remain distinguishable from retrained ones in representation space with output-level success. Our paper extends on this shift toward representation-aware analysis.

\paragraph{The hidden failure mode behind apparent forgetting.} Despite this progress, prior work mainly studies \emph{whether} unlearned representations differ from retrained ones, rather than what kind of mismatch remains at output-level success. This leaves five questions insufficiently characterized, namely (i) whether current methods follow the same transformation as retraining, (ii) whether mismatch differs across forget and retain samples, (iii) whether residual discrepancy is diffuse or directionally organized, (iv) whether the discrepancy is significant empirically and practically, and finally (v) whether the same pattern persists across scale. These questions motivate our paper.

\paragraph{Our novelty.} Our contribution is therefore not another representation metric, but a diagnostic analysis of a hidden failure mode in current \unlearning~evaluation and algorithms. The cited works establish the importance of representation-level evaluation by demonstrating that the representations between unlearning and retraining differ, whereas our work characterizes the structure of that mismatch under a stronger retraining-consistent representation lens. Our novelty lies in the combination of transformation-based structural analysis, retraining-related localization with construction controls, empirical validation against a population of valid retrains and operational consequences evidence.
\section{Discussion}
\label{sec:failure_mode}

\subsection{A Hidden Failure Mode of Current Machine Unlearning}

\textbf{Retraining reveals what output forgetting hides.} Concretely, our results show that current \unlearning~exhibits a hidden failure mode that standard output-level evaluation systematically fails to detect. First, standard output-level success signals are too weak as methods can reflect \emph{apparent forgetting} for successful forgetting while remaining substantially inconsistent with retraining in representation space. Second, this failure is structured rather than random. Under a retraining-consistent representation lens, current methods often partially align with retraining on forget samples, diverge on retain samples, and leave localized residual discrepancy rather than diffuse error. Moreover, these failure modes survive the variability of independent retraining and pose an operational consequence through relearning attacks. Extended discussion and limitations are provided in Appendix~\ref{appendix:extended_failure_mode} and ~\ref{appendix:alternative_views}.

\subsection{Implications for Evaluation and Method Design}
Output-level metrics remain useful, especially in black-box settings, but are too weak to serve as the sole success signal when the goal is to remove the influence of the forget set. Unlearning should therefore be assessed not only by whether a model \emph{looks} forgotten at the output layer, but also by whether its internal changes remain consistent with retraining without the forget data. Our findings further show that optimizing output suppression, forget-set confidence, or output-level membership signals alone can leave structured, asymmetric, and retraining-inconsistent residuals in representation space. This does not imply that current methods are ineffective or fail identically, but suggests that many optimize for \emph{apparent forgetting} rather than retraining-consistent forgetting. Future methods should therefore jointly reduce output-level traces and structured representation discrepancy while more closely approximating the transformations induced by retraining.

\subsection{Implications for Community Practice}
The broader message is not that geometry or CKA are the main contribution, but that these diagnostics reveal hidden structure that standard output-level evaluation misses. Important aspects of unlearning are not fully visible at the output layer: output metrics remain useful, but cannot establish whether unlearning follows retraining in representation space or reveal how residual discrepancies are organized. In this sense, \emph{retraining reveals what output forgetting hides}. The central issue is the field’s default practice that can certify successful unlearning on the basis of a signal that is too weak to measure the goal it claims to measure. Addressing this gap requires progress not only in how \unlearning~is evaluated, but also in how its objectives are defined and its methods are justified.

\section{Conclusion}
\label{sec: conclusion}

This paper presents a diagnostic analysis of a hidden failure mode in current \unlearning: standard output-level evaluations can overestimate successful forgetting while leaving inconsistent representation. Using retraining-consistent representation forgetting as a stronger lens, we show that this mismatch is structured rather than random, with methods often partially aligning with retraining on forget samples but remaining more inconsistent on retain samples and leaving localized residual discrepancy. The mismatch exceeds independent retraining variability and enables information recovery during relearning attacks. Thus, successful unlearning should be judged jointly by output-level metrics and by whether its representations remain consistent with retraining. Retraining reveals what output forgetting hides.

\newpage
{
    \small
    \bibliographystyle{unsrtnat}
    \bibliography{main}
}

\newpage
\appendix

\section{Proof of Theorem~\ref{theo:forget_discrepancy}}
\label{appendix:proof_main_theorem}

\setcounter{theorem}{0}
\setcounter{proposition}{0}

Definition~2 motivates a retraining-consistent representation lens because output agreement alone does not determine the internal state of a model. A model may match the retrained model's predictions while retaining a different representation geometry or following a different representation transformation from the original model. Thus, output forgetting can certify apparent prediction-level behavior, but it cannot by itself certify that the model has moved toward the retrained solution in representation space. The following theorem formalizes this limitation. \theo~\ref{theo:forget_discrepancy} is modest and only serves to justify why output-level alignment alone cannot certify retraining-consistent forgetting.

\begin{theorem}[Output Forgetting Does Not Imply Retraining-Consistent Representation Forgetting]
Let $f_\theta(x)=g_\theta(h_\theta(x))$ be a classifier composed of a feature extractor $h_\theta:\mathcal{X}\to\mathbb{R}^d$ and a prediction head $g_\theta$. There exists an unlearned model $\theta_u$ such that
\begin{equation}
f_{\theta_u}(x)=f_{\theta_r}(x), \qquad \forall x\in\mathcal{D},
\end{equation}
while
\begin{equation}
\mathbb{P}_{\theta_u}^{S}(\mathcal{H}) \neq \mathbb{P}_{\theta_r}^{S}(\mathcal{H})
\end{equation}
for at least one $S\in\{\mathcal{D}_u,\mathcal{D}_r\}$. Hence, perfect output alignment with retraining does not imply representation-level consistency with retraining.
\end{theorem}

\begin{proof}
Consider a linear prediction head $g_\theta(h)=Wh$, where $W\in\mathbb{R}^{k\times d}$, and assume $\ker(W)$ is non-trivial. Let $z_0\in\ker(W)$ with $z_0\neq 0$, and define
\begin{equation}
h_{\theta_u}(x)=
\begin{cases}
h_{\theta_r}(x)+z_0, & x\in\mathcal{D}_u,\\
h_{\theta_r}(x), & x\in\mathcal{D}_r.
\end{cases}
\end{equation}
Then for any $x\in\mathcal{D}$,
\begin{equation}
f_{\theta_u}(x)=Wh_{\theta_u}(x)=Wh_{\theta_r}(x)+Wz_0=Wh_{\theta_r}(x)=f_{\theta_r}(x),
\end{equation}
so the two models are identical in output space on $\mathcal{D}$. However, because $z_0\neq 0$, the representation distribution on $\mathcal{D}_u$ is shifted, and thus
\begin{equation}
\mathbb{P}_{\theta_u}^{\mathcal{D}_u}(\mathcal{H}) \neq \mathbb{P}_{\theta_r}^{\mathcal{D}_u}(\mathcal{H}).
\end{equation}
Therefore, output-level alignment does not guarantee retraining-consistent representation forgetting.
\end{proof}

\theo~\ref{theo:forget_discrepancy} shows that a model can appear to forget at the output layer while still remaining inconsistent with retraining in representation space. This is precisely why standard output-level evaluation can overestimate successful unlearning.

\paragraph{A stronger operational variant.}
Theorem~\ref{theo:forget_discrepancy} shows that output agreement does not imply representation agreement. The following proposition gives an operational form of this argument, showing that residual discrepancy may remain visible after a linear projection in representation space even when outputs match.

\begin{proposition}[Output Agreement Does Not Rule Out Linearly Visible Representation-Level Residuals]
\label{prop:linear_visible_residual}
There exist an unlearned model $\theta_u$ and a retrained model $\theta_r$ such that
\begin{equation}
f_{\theta_u}(x)=f_{\theta_r}(x), \qquad \forall x\in\mathcal{D},
\end{equation}
while there exists a linear functional $\ell(h)=w^\top h+b$ for which the induced one-dimensional distributions of $\ell(h_{\theta_u}(x))$ and $\ell(h_{\theta_r}(x))$ differ on $\mathcal{D}_u$. Hence, exact output agreement does not preclude linearly visible residual discrepancy in representation space.
\end{proposition}

\begin{proof}[Proof]
Use the same construction as in \theo~\ref{theo:forget_discrepancy}. Let
\[
h_{\theta_u}(x)=
\begin{cases}
h_{\theta_r}(x)+z_0, & x\in\mathcal{D}_u,\\
h_{\theta_r}(x), & x\in\mathcal{D}_r,
\end{cases}
\]
where $z_0\neq 0$ lies in $\ker(W)$ for the linear prediction head $g_\theta(h)=Wh$. Then
\[
f_{\theta_u}(x)=f_{\theta_r}(x), \qquad \forall x\in\mathcal{D},
\]
exactly as in \theo~\ref{theo:forget_discrepancy}.

Now choose any $w\in\mathbb{R}^d$ such that $w^\top z_0\neq 0$; such a vector exists because $z_0\neq 0$. Define the linear functional
\[
\ell(h)=w^\top h+b
\]
for any scalar $b$. Then for every $x\in\mathcal{D}_u$,
\[
\ell(h_{\theta_u}(x))
=
w^\top(h_{\theta_r}(x)+z_0)+b
=
\ell(h_{\theta_r}(x)) + w^\top z_0.
\]
Since $w^\top z_0\neq 0$, the projected representations on $\mathcal{D}_u$ differ by a nonzero constant shift, and therefore their induced one-dimensional distributions are not equal. Hence, exact output agreement does not preclude residual discrepancy that remains visible after a linear projection.
\end{proof}

\section{Robustness of Representation Layer and Similarity Measure}
\label{appendix:rep_metric_layer}

The main paper performs all representation diagnostics using the penultimate representation layer and operationalizes representation similarity using Centered Kernel Alignment (CKA) between the evaluated model and a reference model. This appendix extends the analysis in two directions: evaluating representations at different layers and measuring representation similarity using Singular Vector Canonical Correlation Analysis (SVCCA) as an alternative similarity measure. The goal is to assess whether the penultimate layer is the most suitable layer for representation forgetting diagnosis and whether the representation mismatch identified in the main paper can also be captured using a similarity measure other than CKA. POUR-P is not tested in this section because the algorithm itself projects only the penultimate layer without modifying the earlier layers, so it is not a meaningful comparison for this purpose.

\subsection{Representation Layer Selection}

In the primary setting with ResNet-18 on CIFAR-10, the model contains four intermediate representation layers. We select the penultimate layer for the main representation diagnostics because it is directly consumed by the classifier, providing the most direct connection between output-level and representation-level forgetting. We additionally conduct a sensitivity analysis using earlier representation layers, evaluating both CKA and cosine directional alignment to examine how different unlearning methods compare with the retrained model at different depths. A held-out retrained model is used as the reference to capture variation caused by independent retraining.

\fig~\ref{fig:cka_layers} illustrates the forget-set CKA, $\mathrm{CKA}_u$, and retain-set CKA, $\mathrm{CKA}_r$, between the unlearned model and the retrained model across the four intermediate representation layers. The CKA at the first three layers are highly compressed: even the original model, which has undergone no unlearning, achieves CKA values of approximately 0.9. The two exceptions are Boundary Shrink and UNSIR, which exhibit destructive representation collapse beginning at earlier layers. In contrast, the penultimate layer provides substantially greater separation between methods in $\mathrm{CKA}_u$.

The cosine directional results in \fig~\ref{fig:cosine_layers} exhibit a similar dependence on representation depth, with cosine alignment remaining close to zero in the first three layers. At the penultimate layer, the retain-side cosine similarities become more dispersed, while the forget-side directional alignment becomes more pronounced.

We therefore make the more precise claim that the penultimate layer exhibits the clearest representation structure among the examined layers, rather than claiming that the effect is equally strong throughout the network. Since these results use a single held-out retrained reference, we treat this analysis as a qualitative sensitivity study rather than as a separately calibrated detection test.

\begin{figure}[t]
    \centering
    \begin{minipage}{\textwidth}
        \centering
        \begin{subfigure}{0.48\textwidth}
            \centering
            \includegraphics[width=\textwidth]{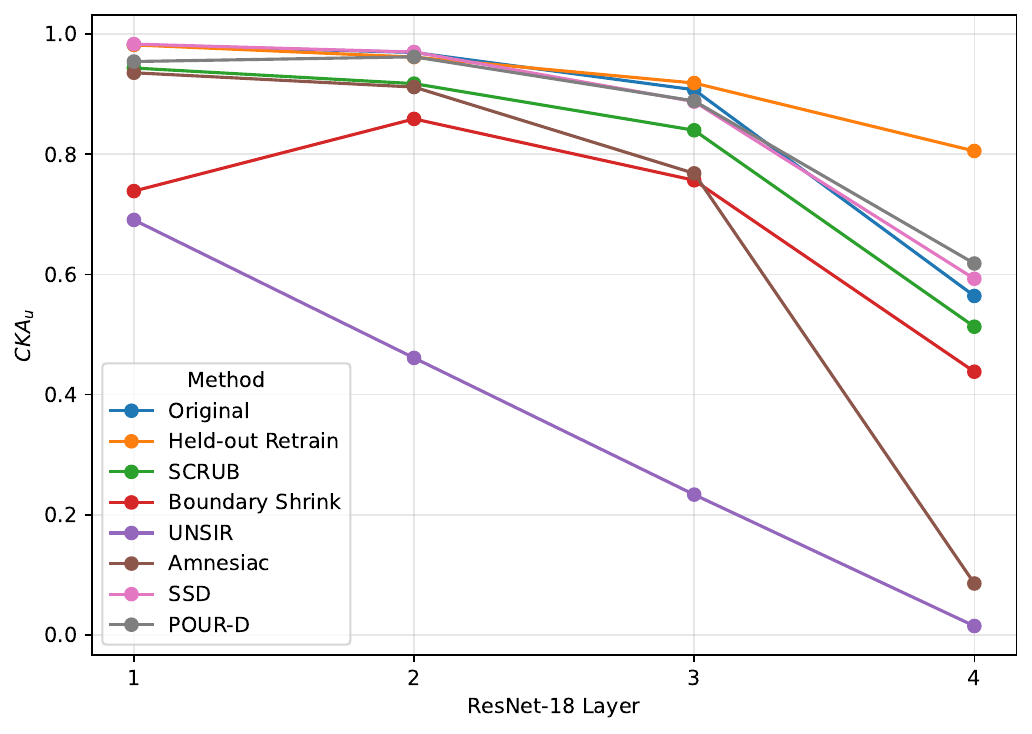}
            \captionsetup{justification=centering}
            \caption{$\mathrm{CKA}_u$}
            \label{fig:cka_u_layers}
        \end{subfigure}
        \hfill
        \begin{subfigure}{0.48\textwidth}
            \centering
            \includegraphics[width=\textwidth]{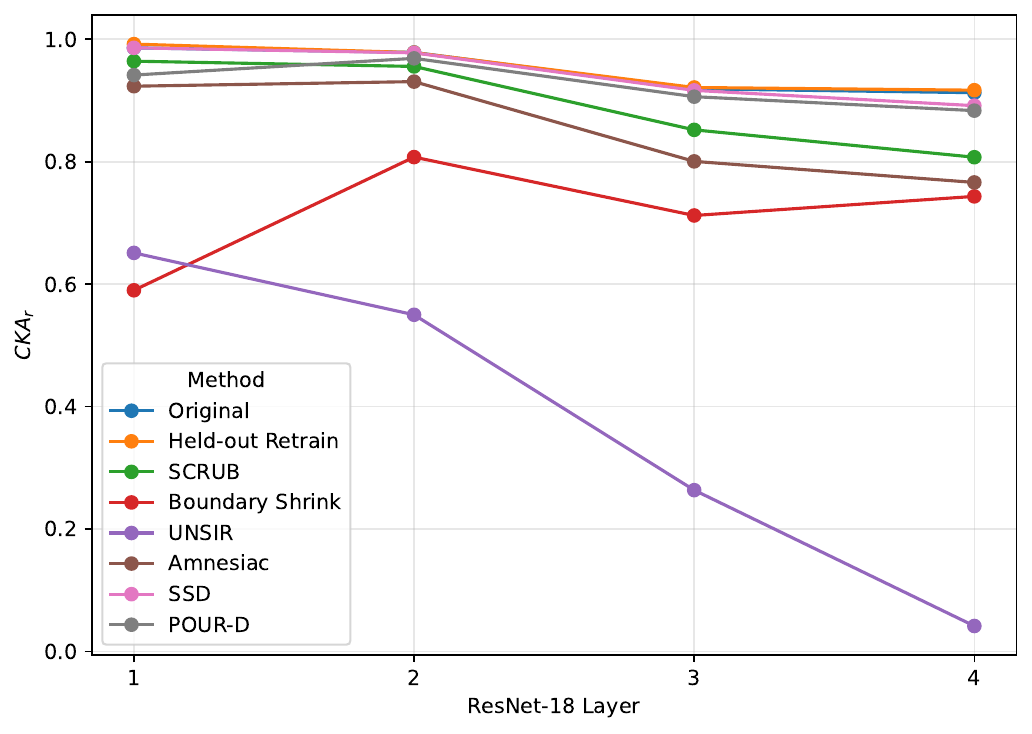}
            \captionsetup{justification=centering}
            \caption{$\mathrm{CKA}_r$}
            \label{fig:cka_r_layers}
        \end{subfigure}
        \caption{CKA values to retrained reference across layers.}
        \label{fig:cka_layers}
    \end{minipage}
\end{figure}
\begin{figure}[t]
    \centering
    \begin{minipage}{\textwidth}
        \centering
        \begin{subfigure}{0.48\textwidth}
            \centering
            \includegraphics[width=\textwidth]{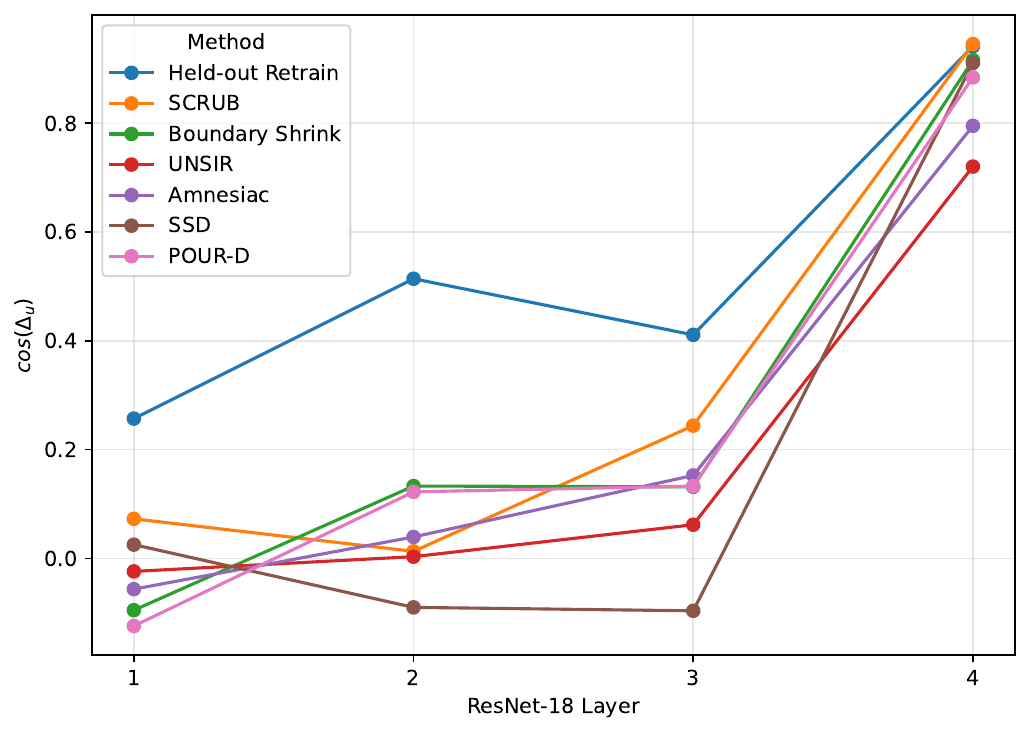}
            \captionsetup{justification=centering}
            \caption{$\mathrm{cos}(\Delta_u)$}
            \label{fig:cosine_u_layers}
        \end{subfigure}
        \hfill
        \begin{subfigure}{0.48\textwidth}
            \centering
            \includegraphics[width=\textwidth]{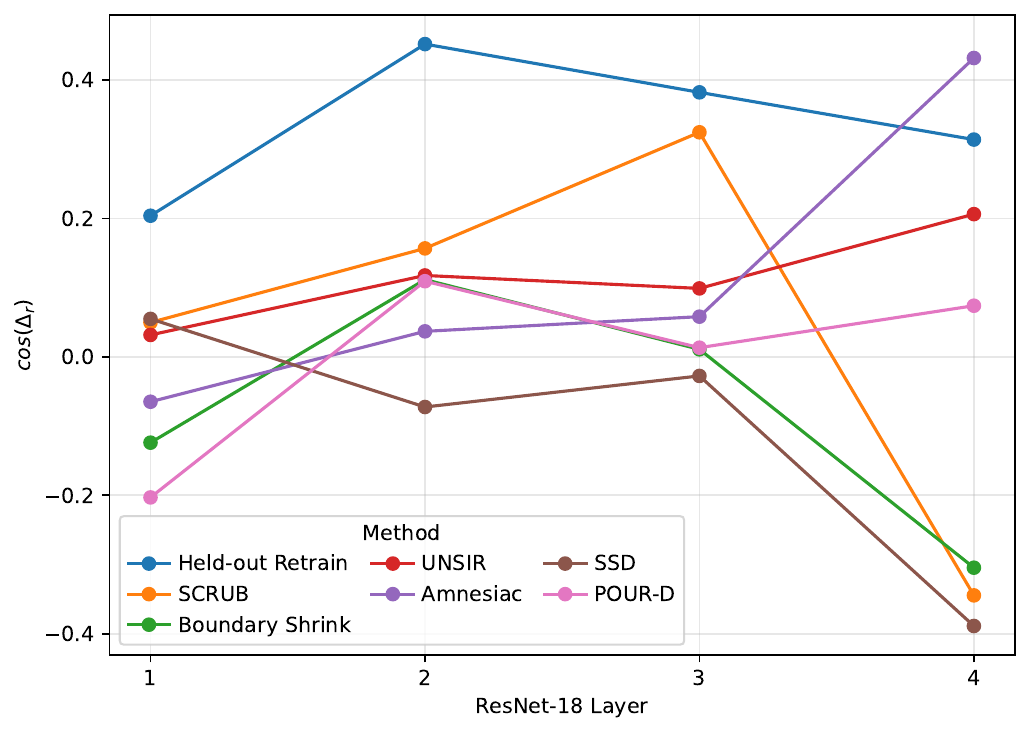}
            \captionsetup{justification=centering}
            \caption{$\mathrm{cos}(\Delta_r)$}
            \label{fig:cosine_r_layers}
        \end{subfigure}
        \caption{Cosine directional alignment to retrained reference across layers.}
        \label{fig:cosine_layers}
    \end{minipage}
\end{figure}
\begin{figure}[t]
    \centering
    \begin{minipage}{\textwidth}
        \centering
        \begin{subfigure}{0.48\textwidth}
            \centering
            \includegraphics[width=\textwidth]{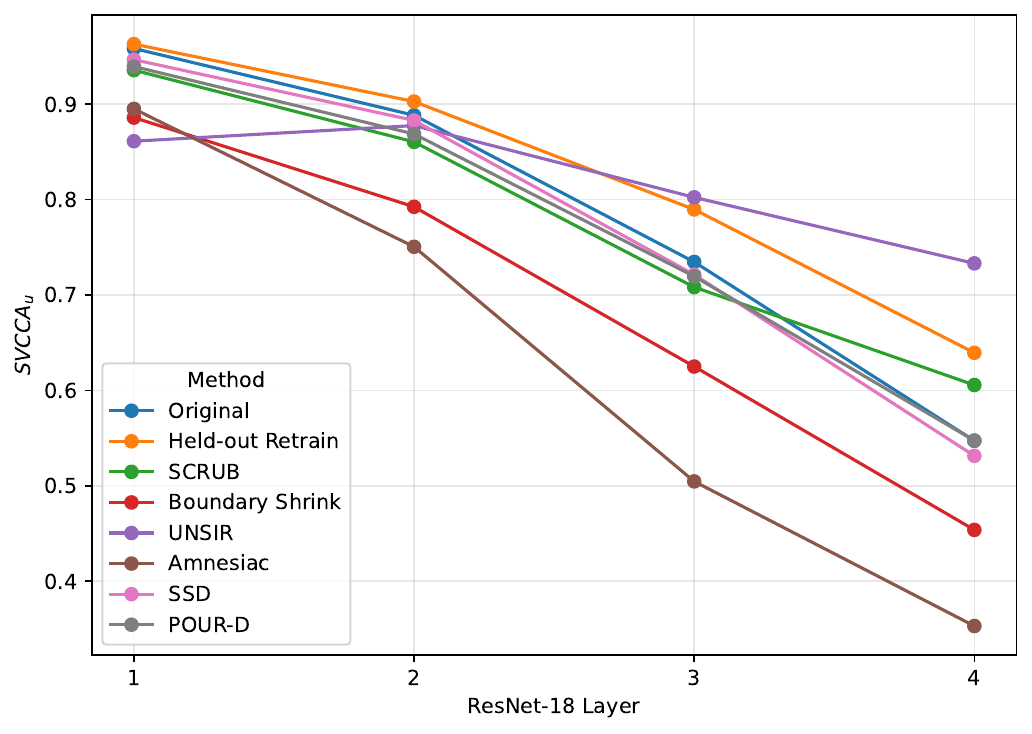}
            \captionsetup{justification=centering}
            \caption{$\mathrm{SVCCA}_u$}
            \label{fig:svcca_u_layers}
        \end{subfigure}
        \hfill
        \begin{subfigure}{0.48\textwidth}
            \centering
            \includegraphics[width=\textwidth]{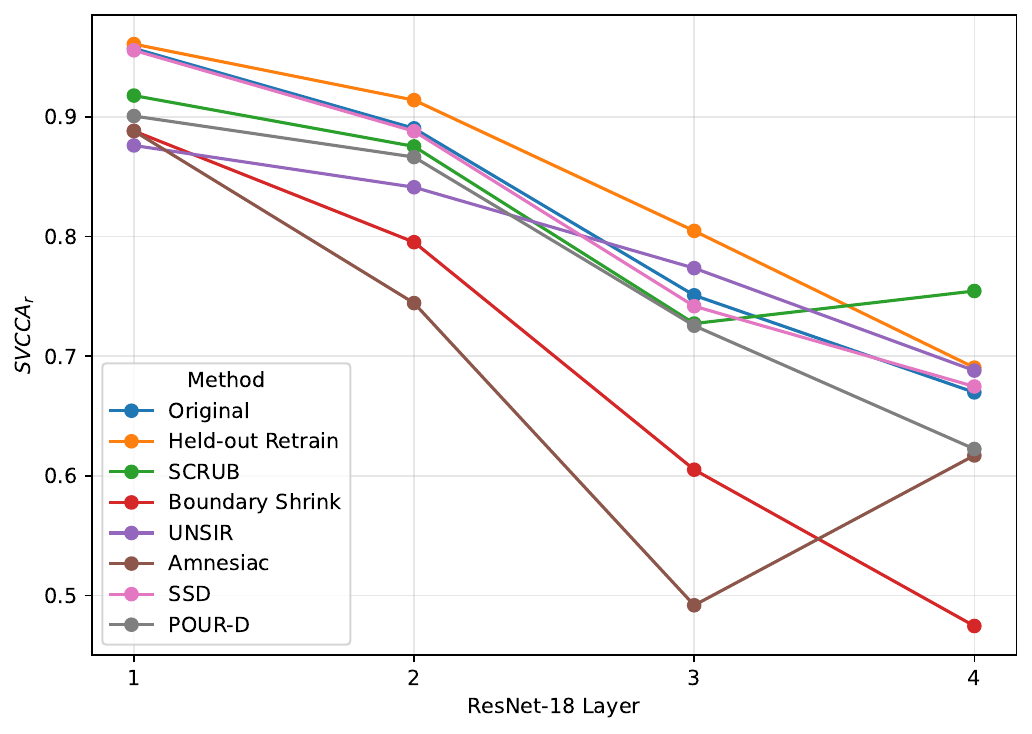}
            \captionsetup{justification=centering}
            \caption{$\mathrm{SVCCA}_r$}
            \label{fig:svcca_r_layers}
        \end{subfigure}
        \caption{SVCCA values to retrained reference across layers.}
        \label{fig:svcca_layers}
    \end{minipage}
\end{figure}

\subsection{Alternative Similarity Measure}

CKA is not the only measure that can be used to quantify similarity between neural representations. Singular Vector Canonical Correlation Analysis (SVCCA) provides an alternative approach for comparing internal representations across layers. Following~\cite{kornblith2019similarity}, the SVCCA score can be expressed using the nuclear norm of the cross-correlation between the retained SVD components:

\begin{equation}
\operatorname{SVCCA}(X,Y) = \frac{\left\|(U_YT_Y)^\top(U_XT_X)\right\|_*} {\min(\|T_X\|_F^2,\|T_Y\|_F^2)}
\end{equation}

where $(U_X,U_Y)$ denote the left singular vectors and $(T_X,T_Y)$ select the retained singular directions according to a specified explained-variance threshold.

CKA measures similarity between the Gram matrices of two representations, whereas SVCCA first applies SVD to reduce each representation to the principal directions explaining a specified fraction of the variance, followed by CCA on the reduced representations. Both CKA and SVCCA are invariant to orthogonal transformations and isotropic scaling. However, CKA is not invariant to general invertible linear transformations, while SVCCA is invariant to such transformations only when the retained subspace remains unchanged~\cite{kornblith2019similarity}. Despite its use as an established representation similarity measure, SVCCA has largely been superseded by CKA for comparing representations across independently initialized networks because CKA can more reliably identify correspondences between such networks.

We therefore recompute the layer-wise representation similarity using SVCCA. We use the mean canonical correlation as the SVCCA similarity score and compute on forget-set, $\mathrm{SVCCA}_u$ and retain-set, $\mathrm{SVCCA}_r$ separately. In \fig~\ref{fig:svcca_layers}, SVCCA shows a similar layer profile to CKA with compressed similarity in earlier layers and becomes more differentiated at the penultimate layer. However, the held-out retrained model has lower $\mathrm{SVCCA}_u$ and $\mathrm{SVCCA}_r$ at the final two layers than the corresponding CKA values. As a result, unlearned models appear more similar to the held-out retrained reference under SVCCA than under CKA.

This observation is consistent with the results in~\cite{kornblith2019similarity}, which show that CKA can reveal consistent relationships between representations trained from different random initializations, while SVCCA cannot. Hence, CKA is more suitable for our retraining-consistent representation forgetting diagnosis where representation similarity must be interpreted relative to the variation induced by stochastic retraining, as it preserves stronger similarity signals across independent retrained models.

\section{Experimental Details}
\label{appendix:experimental_details}

This appendix provides full experimental details supporting the results reported in the main paper. We describe the datasets used across all settings (\tab~\ref{tab:dataset_details}), the hyperparameters used to train the original and retrained baseline models (\tab~\ref{tab:baseline_hparams}), and the per-method unlearning hyperparameters for each of the five dataset / architecture combinations studied: CIFAR-10 / ResNet-18, CIFAR-100 / ResNet-18, CIFAR-100 / ResNet-50, TinyImageNet / ResNet-18, and CIFAR-100 / ViT-Tiny (\tab~\ref{tab:hparams_scrub},~\ref{tab:hparams_amnesiac},~\ref{tab:hparams_boundary},~\ref{tab:hparams_unsir},~\ref{tab:hparams_ssd},~\ref{tab:hparams_pourd}). Dashes (---) indicate the entry is not applicable for that setting. All experiments were conducted on NVIDIA A100 GPUs.

\begin{table}[H]
\centering
\small
\renewcommand{\arraystretch}{1.1}
\begin{tabular}{l|ccc}
\toprule
\textbf{Dataset Statistics} & \textbf{CIFAR-10} & \textbf{CIFAR-100} & \textbf{TinyImageNet} \\
\midrule
Number of classes    & 10     & 100    & 200     \\
Training images      & 50,000 & 50,000 & 100,000 \\
Validation images    & 10,000 & 10,000 & 10,000  \\
Forget-class images  & 5,000  & 500    & 500     \\
Image size           & $32 \times 32$ & $32 \times 32$ & $64 \times 64$ \\
\bottomrule
\end{tabular}
\caption{Dataset statistics for CIFAR-10, CIFAR-100 and TinyImageNet.}
\label{tab:dataset_details}
\end{table}

\begin{table}[H]
\centering
\small
\renewcommand{\arraystretch}{1.2}
\resizebox{\textwidth}{!}{%
\begin{tabular}{l|ccccc}
\toprule
\textbf{Hyperparameter}
  & \makecell{\textbf{CIFAR-10} \\ \textbf{ResNet-18}}
  & \makecell{\textbf{CIFAR-100} \\ \textbf{ResNet-18}}
  & \makecell{\textbf{CIFAR-100} \\ \textbf{ResNet-50}}
  & \makecell{\textbf{TinyImageNet} \\ \textbf{ResNet-18}}
  & \makecell{\textbf{CIFAR-100} \\ \textbf{ViT-Tiny}} \\
\midrule
Epochs              & 50    & 200   & 200   & 200   & 300   \\
Optimizer           & SGD   & SGD   & SGD   & SGD   & AdamW \\
Batch size          & 128   & 128   & 128   & 256   & 256   \\
Learning rate       & 0.01  & 0.1   & 0.1   & 0.01  & $3\times10^{-4}$ \\
Momentum            & 0.9   & 0.9   & 0.9   & 0.9   & ---   \\
Weight decay        & $5\times10^{-4}$ & $2.5\times10^{-3}$ & $5\times10^{-4}$ & $5\times10^{-4}$ & 0.07 \\
Label smoothing     & 0     & 0     & 0     & 0     & 0.1   \\
Scheduler           & ---   & \makecell{ReduceLR\\OnPlateau} & \makecell{ReduceLR\\OnPlateau} & \makecell{ReduceLR\\OnPlateau} & \makecell{CosineAnnealing\\LR} \\
Scheduler LR factor & ---   & 0.1   & 0.1   & 0.1   & ---   \\
Scheduler patience  & ---   & 5     & 5     & 5     & ---   \\
Min LR              & ---   & ---   & ---   & ---   & $10^{-6}$ \\
Warmup epochs       & ---   & ---   & ---   & ---   & 10    \\
Early stopping      & \checkmark & \checkmark & \checkmark & \checkmark & ---  \\
Pretrained weight   & $^\dagger$ & --- & --- & \checkmark & --- \\
\bottomrule
\end{tabular}}
\caption{Hyperparameters for original and retrained baseline training. $^\dagger$CIFAR-10 / ResNet-18 was initialised from weights pretrained on CIFAR-100 for 30 epochs at learning rate 0.1.}
\label{tab:baseline_hparams}
\end{table}

\begin{table}[H]
\centering
\small
\renewcommand{\arraystretch}{1.2}
\resizebox{\textwidth}{!}{%
\begin{tabular}{l|ccccc}
\toprule
\textbf{Hyperparameter}
  & \makecell{\textbf{CIFAR-10} \\ \textbf{ResNet-18}}
  & \makecell{\textbf{CIFAR-100} \\ \textbf{ResNet-18}}
  & \makecell{\textbf{CIFAR-100} \\ \textbf{ResNet-50}}
  & \makecell{\textbf{TinyImageNet} \\ \textbf{ResNet-18}}
  & \makecell{\textbf{CIFAR-100} \\ \textbf{ViT-Tiny}} \\
\midrule
Gamma           & 0.99  & 1     & 1     & 1     & 1     \\
Alpha           & 0.001 & 0.5 & 0.5 & 0.5 & 0.5   \\
Max steps       & 2     & 5     & 5     & 5     & 5     \\
Min steps       & 3     & 5     & 5     & 4     & 5     \\
Optimizer       & SGD   & Adam  & Adam  & Adam  & Adam  \\
Learning rate   & $5\times10^{-4}$ & $5\times10^{-4}$ & $5\times10^{-4}$ & $5\times10^{-4}$ & $5\times10^{-4}$ \\
LR decay rate   & 0.1 & 0.1 & 0.1 & 0.1 & 0.1 \\
LR decay epochs & {[}3, 5, 9{]} & {[}2{]} & {[}2{]} & {[}2{]} & {[}2{]} \\
Momentum        & 0.9   & ---   & ---   & ---   & ---   \\
Weight decay    & $5\times10^{-4}$ & $5\times10^{-4}$ & $5\times10^{-4}$ & $5\times10^{-4}$ & $5\times10^{-4}$ \\
\bottomrule
\end{tabular}}
\caption{Hyperparameters for SCRUB.}
\label{tab:hparams_scrub}
\end{table}

\begin{table}[H]
\centering
\small
\renewcommand{\arraystretch}{1.2}
\resizebox{\textwidth}{!}{%
\begin{tabular}{l|ccccc}
\toprule
\textbf{Hyperparameter}
  & \makecell{\textbf{CIFAR-10} \\ \textbf{ResNet-18}}
  & \makecell{\textbf{CIFAR-100} \\ \textbf{ResNet-18}}
  & \makecell{\textbf{CIFAR-100} \\ \textbf{ResNet-50}}
  & \makecell{\textbf{TinyImageNet} \\ \textbf{ResNet-18}}
  & \makecell{\textbf{CIFAR-100} \\ \textbf{ViT-Tiny}} \\
\midrule
Optimizer       & Adam  & Adam  & Adam  & Adam  & Adam  \\
Batch size      & 64    & 64    & 64    & 64    & 64    \\
Epochs          & 5     & 5     & 5     & 3     & 5     \\
Learning rate   & $10^{-3}$ & $10^{-3}$ & $10^{-3}$ & $10^{-4}$ & $10^{-3}$ \\
\bottomrule
\end{tabular}}
\caption{Hyperparameters for Amnesiac.}
\label{tab:hparams_amnesiac}
\end{table}

\begin{table}[H]
\centering
\small
\renewcommand{\arraystretch}{1.2}
\resizebox{\textwidth}{!}{%
\begin{tabular}{l|ccccc}
\toprule
\textbf{Hyperparameter}
  & \makecell{\textbf{CIFAR-10} \\ \textbf{ResNet-18}}
  & \makecell{\textbf{CIFAR-100} \\ \textbf{ResNet-18}}
  & \makecell{\textbf{CIFAR-100} \\ \textbf{ResNet-50}}
  & \makecell{\textbf{TinyImageNet} \\ \textbf{ResNet-18}}
  & \makecell{\textbf{CIFAR-100} \\ \textbf{ViT-Tiny}} \\
\midrule
Optimizer       & SGD   & SGD   & SGD   & SGD   & SGD   \\
Eps             & 0.1   & 0.05  & 0.05  & 0.1   & 0.3   \\
Poison epochs   & 10    & 10    & 10    & 5     & 10    \\
Fine-tune LR    & $10^{-5}$ & $10^{-5}$ & $10^{-5}$ & $10^{-5}$ & $10^{-4}$ \\
Momentum        & 0.9   & 0.9   & 0.9   & 0.9   & 0.9   \\
\bottomrule
\end{tabular}}
\caption{Hyperparameters for Boundary Shrink.}
\label{tab:hparams_boundary}
\end{table}

\begin{table}[H]
\centering
\small
\renewcommand{\arraystretch}{1.2}
\resizebox{\textwidth}{!}{%
\begin{tabular}{l|ccccc}
\toprule
\textbf{Hyperparameter}
  & \makecell{\textbf{CIFAR-10} \\ \textbf{ResNet-18}}
  & \makecell{\textbf{CIFAR-100} \\ \textbf{ResNet-18}}
  & \makecell{\textbf{CIFAR-100} \\ \textbf{ResNet-50}}
  & \makecell{\textbf{TinyImageNet} \\ \textbf{ResNet-18}}
  & \makecell{\textbf{CIFAR-100} \\ \textbf{ViT-Tiny}} \\
\midrule
Noise batch size          & 256   & 256   & 256   & 16    & 256   \\
Samples per retain class  & 1000  & 1000  & 1000  & 450   & 1000 \\
Noise epochs              & 40    & 40    & 40    & 20    & 40    \\
Noise LR                  & 0.1   & 0.1   & 0.1   & 0.01  & 0.1 \\
Noise $\ell_2$ lambda     & 0.1   & 0.1   & 0.1   & 0.2   & 0.1   \\
Noise copies              & 20    & 20    & 20    & 15    & 20    \\
Impair epochs             & 1     & 1     & 1     & 6     & 1     \\
Impair LR                 & 0.02  & 0.02  & 0.02  & 0.003 & $10^{-5}$ \\
Repair epochs             & 1     & 1     & 1     & 8     & 1     \\
Repair LR                 & 0.01  & 0.01  & 0.01  & 0.008 & $3\times10^{-4}$ \\
Repair weight decay       & 0     & 0     & 0     & 0     & 0.05  \\
Optimizer                 & Adam  & Adam  & Adam  & Adam  & Adam  \\
\bottomrule
\end{tabular}}
\caption{Hyperparameters for UNSIR.}
\label{tab:hparams_unsir}
\end{table}

\begin{table}[H]
\centering
\small
\renewcommand{\arraystretch}{1.2}
\resizebox{\textwidth}{!}{%
\begin{tabular}{l|ccccc}
\toprule
\textbf{Hyperparameter}
  & \makecell{\textbf{CIFAR-10} \\ \textbf{ResNet-18}}
  & \makecell{\textbf{CIFAR-100} \\ \textbf{ResNet-18}}
  & \makecell{\textbf{CIFAR-100} \\ \textbf{ResNet-50}}
  & \makecell{\textbf{TinyImageNet} \\ \textbf{ResNet-18}}
  & \makecell{\textbf{CIFAR-100} \\ \textbf{ViT-Tiny}} \\
\midrule
Lambda  & 1  & 1  & 1  & 1  & 1  \\
Alpha   & 10 & 10 & 10 & 50 & 50 \\
\bottomrule
\end{tabular}}
\caption{Hyperparameters for SSD.}
\label{tab:hparams_ssd}
\end{table}

\begin{table}[H]
\centering
\small
\renewcommand{\arraystretch}{1.2}
\resizebox{\textwidth}{!}{%
\begin{tabular}{l|ccccc}
\toprule
\textbf{Hyperparameter}
  & \makecell{\textbf{CIFAR-10} \\ \textbf{ResNet-18}}
  & \makecell{\textbf{CIFAR-100} \\ \textbf{ResNet-18}}
  & \makecell{\textbf{CIFAR-100} \\ \textbf{ResNet-50}}
  & \makecell{\textbf{TinyImageNet} \\ \textbf{ResNet-18}}
  & \makecell{\textbf{CIFAR-100} \\ \textbf{ViT-Tiny}} \\
\midrule
Epochs        & 10    & 100   & 100   & 100   & 10    \\
Optimizer     & Adam  & Adam  & Adam  & Adam  & Adam  \\
Learning rate & $10^{-4}$ & $10^{-4}$ & $10^{-4}$ & $10^{-4}$ & $10^{-4}$ \\
\bottomrule
\end{tabular}}
\caption{Hyperparameters for POUR-D.}
\label{tab:hparams_pourd}
\end{table}

\section{Additional Representation-Space Sanity Checks}
\label{appendix:linear_probe}

The main paper diagnoses representation-level forgetting through directional alignment, subspace decomposition, and centered kernel alignment. This appendix adds two complementary sanity checks, namely linear probing and t-SNE visualization, to test whether the same gap remains visible under simpler readouts of the learned representation. The goal is not to introduce new primary metrics, but to verify that the main diagnosis is not specific to the particular representation-level diagnostics used in the paper.

\subsection{Linear Probing}

We first test whether forget-class information remains linearly recoverable from the learned representations after unlearning. For each model, we freeze the backbone and train a single linear classification layer on the full dataset $\mathcal{D}$, i.e., $\mathcal{D}_u \cup \mathcal{D}_r$, using the original class labels. We then report forget-set probe accuracy $\mathrm{Acc}_u^{\mathrm{probe}}$ and retain-set probe accuracy $\mathrm{Acc}_r^{\mathrm{probe}}$. These quantities measure, respectively, how much class-identifying structure for the forget class remains linearly accessible in the backbone representation and how well retain-class structure is preserved.

The intuition is simple. If a backbone no longer encodes forget-class features in a linearly separable way, then even a linear head trained on the original labels should fail to recover that class from $\mathcal{D}_u$. Conversely, high $\mathrm{Acc}_u^{\mathrm{probe}}$ indicates that forget-class identity remains recoverable from the representation, even if output-level forgetting appears successful. All experiments are conducted on CIFAR-10 with ResNet-18. The linear probe is trained with SGD for 10 epochs using a learning rate of $1\times10^{-3}$ across all seven unlearning methods, together with the original and retrained baselines.

\begin{table}[ht]
\centering
\scriptsize
\setlength{\tabcolsep}{3.5pt}
\renewcommand{\arraystretch}{1.05}
\begin{tabular}{l|cc|cc|p{3.95cm}}
\toprule
\multirow{2}{*}{\textbf{Method}} 
    & \multicolumn{2}{c|}{\textbf{Linear Probe}} 
    & \multicolumn{2}{c|}{\textbf{Output-level}} 
    & \multirow{2}{*}{\centering \textbf{Reading}} \\
\cmidrule(lr){2-3} \cmidrule(lr){4-5}
& $\mathrm{Acc}_u^{\mathrm{probe}}\!\downarrow$ 
& $\mathrm{Acc}_r^{\mathrm{probe}}\!\uparrow$ 
& $\mathrm{Acc}_u\!\downarrow$
& $\mathrm{Acc}_r\!\uparrow$ \\
\midrule
Original           & 98.20 & 98.11 & 98.65 & 98.35 & Original representation \\
\textbf{Retrain}   & \textbf{45.14} & \textbf{98.26} & \textbf{0.00} & \textbf{98.74} & Reference for forgetting \\
\midrule
SCRUB              & 46.02 & 98.08 & 0.00 & 98.98 & Closest to retrain under probe \\
Boundary Shrink    & 49.38 & 85.19 & 13.55 & 84.27 & Reduced recoverability, damaged retain \\
UNSIR              & 54.18 & 23.91 & 0.00 & 27.55 & Destructive forgetting \\
Amnesiac           & 74.68 & 97.45 & 0.00 & 97.87 & Output forgotten, probe recoverable \\
SSD                & 70.24 & 97.37 & 0.00 & 97.39 & Output forgotten, probe recoverable \\
POUR-P             & 83.38 & 98.10 & 0.00 & 98.36 & Output forgotten, probe recoverable \\
POUR-D             & 70.24 & 96.50 & 0.02 & 96.90 & Output forgotten, probe recoverable \\
\bottomrule
\end{tabular}
\caption{Linear probe class-recovery accuracy and output-level accuracy on CIFAR-10 with ResNet-18. Low output forget accuracy alone does not imply representation behavior close to retraining.}
\label{tab:linear_probe}
\vspace{-8pt}
\end{table}

Table~\ref{tab:linear_probe} supports the main paper's diagnosis. The original model attains near-perfect $\mathrm{Acc}_u^{\mathrm{probe}}$ ($98.2\%$), confirming that forget-class identity is fully encoded in its representation. In contrast, the retrained model achieves only $45.1\%$, despite the probe being trained on the original class labels. This indicates that retraining no longer maintains the forget class as linearly separable in the backbone representation.

Among methods that achieve complete output-level forgetting, i.e., $\mathrm{Acc}_u=0$, several still retain strong linear recoverability of the forget class. In particular, POUR-P ($83.4\%$), Amnesiac ($74.7\%$), SSD ($70.2\%$), and POUR-D ($70.2\%$) all remain far above retraining in $\mathrm{Acc}_u^{\mathrm{probe}}$. Thus, although these methods appear successful under output-level evaluation, forget-class identity remains linearly accessible in their representations.

Boundary Shrink and UNSIR produce $\mathrm{Acc}_u^{\mathrm{probe}}$ values closer to retraining ($49.4\%$ and $54.2\%$, respectively), but this comes with severe degradation in $\mathrm{Acc}_r^{\mathrm{probe}}$ ($85.2\%$ and $23.9\%$), mirrored by their poor $\mathrm{Acc}_r$. This pattern is more consistent with broad representational disruption than with selective forgetting.

SCRUB is the strongest case under linear probing. Its $\mathrm{Acc}_u^{\mathrm{probe}}$ ($46.0\%$) is close to retraining, while $\mathrm{Acc}_r^{\mathrm{probe}}$ ($98.1\%$) remains essentially fully preserved. Under this probe-based view, SCRUB appears close to retraining. However, the main paper shows that this agreement is incomplete, with SCRUB remaining directionally misaligned with retraining and geometrically inconsistent under CKA. Moreover, its output-level forgetting degrades under larger datasets and models (\tab~\ref{tab:cls_cifar100_r50} and \tab~\ref{tab:cls_tinyimagenet_r18}). Thus, linear probing provides a useful sanity check, but not a complete account of retraining consistency.

\begin{figure}[t]
    \centering
    \includegraphics[width=\linewidth]{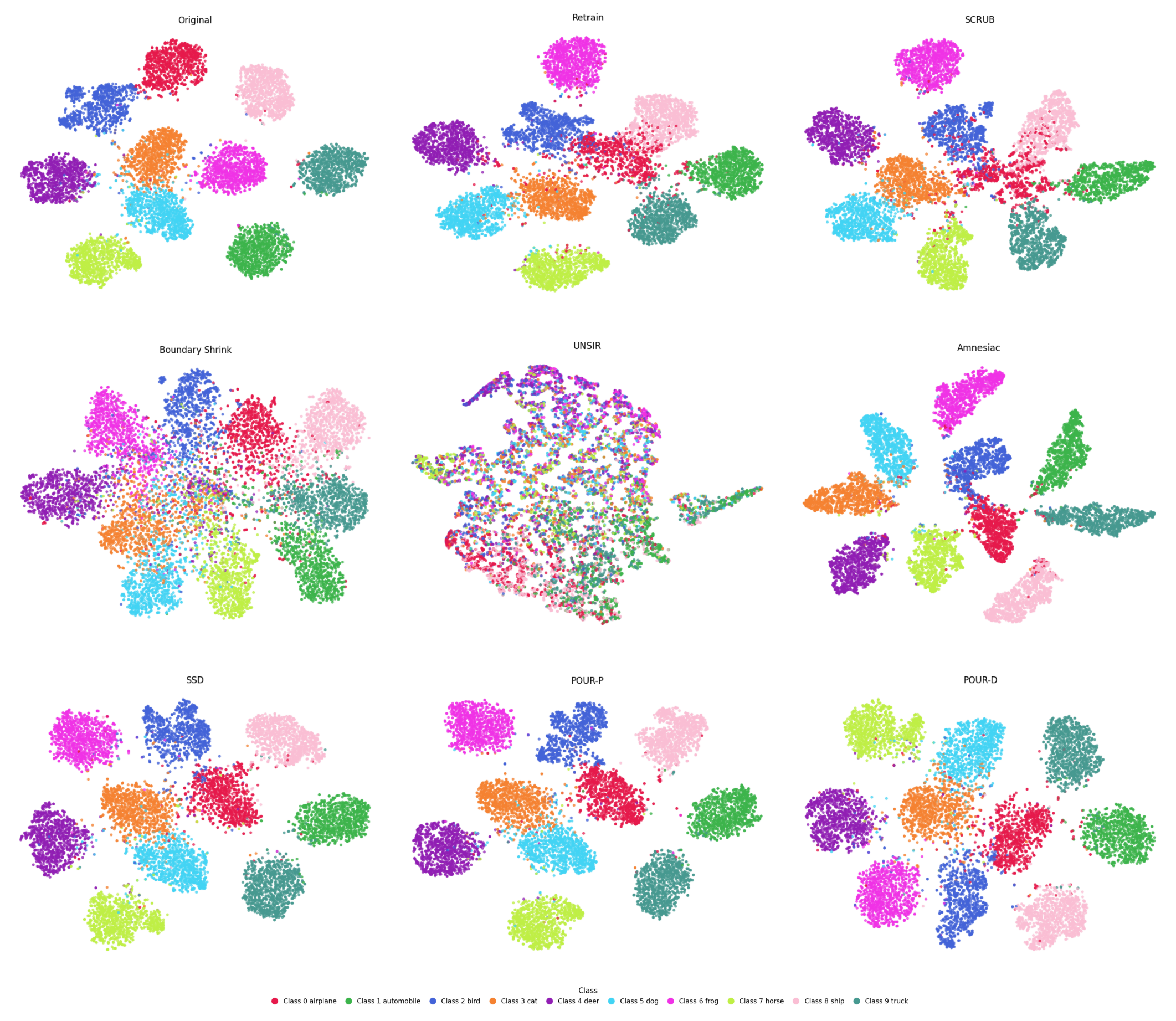}
    \caption{t-SNE visualization of feature representations on CIFAR-10 with ResNet-18. The forget class is Class 0 (airplane), shown in red.}
    \label{fig:tsne_merged}
    \vspace{-10pt}
\end{figure}

\subsection{t-SNE Visualization}

We next provide a qualitative sanity check through t-SNE visualization~\cite{van2008visualizing}. \fig~\ref{fig:tsne_merged} visualizes the learned representations of $\mathcal{D}_u$ and $\mathcal{D}_r$ for each model. We stress that t-SNE is used here only as an illustrative complement to the quantitative diagnostics above; it is not a formal basis for the paper's claims.

The visual patterns broadly agree with the linear-probe results. Amnesiac, SSD, POUR-P, and POUR-D retain a visually distinct forget-class cluster that remains separated from the retain-class representations. This resembles the original model more than the retrained model and is consistent with their high $\mathrm{Acc}_u^{\mathrm{probe}}$, indicating that forget-class geometry remains substantially intact despite complete output-level forgetting.

Boundary Shrink and UNSIR show a different behavior. Their forget-class cluster becomes diffuse, but retain-class clusters are also visibly distorted relative to the original model. This aligns with their degraded $\mathrm{Acc}_r^{\mathrm{probe}}$ and $\mathrm{Acc}_r$, suggesting broad representational disruption rather than selective forgetting.

SCRUB appears visually closest to retraining, as its forget cluster is diffuse while its retain clusters remain relatively well separated. This agrees with its near-retraining linear-probe behavior. However, as shown in the main paper, this visual similarity does not extend to directional alignment or CKA geometry, and it does not persist under increased dataset complexity or model size.

In summary, linear probing and t-SNE support the same qualitative conclusion as the main paper. Even under simpler and more intuitive representation-space diagnostics, existing unlearning methods do not faithfully reproduce the representation-level behavior of retraining, and output-level forgetting alone remains too weak to reveal this gap.

\section{Additional Directional Diagnostics}
\label{appendix:shift_mag_ratio}

The main paper characterizes forget/retain asymmetry through the directional alignment between unlearning and retraining shifts. This appendix adds a complementary diagnostic for the CIFAR-10 / ResNet-18 baseline by comparing not only the direction, but also the magnitude of those shifts relative to retraining. The goal is to test whether apparent directional agreement is accompanied by retraining-like adjustment strength, or whether unlearning under- or over-adjusts even when the direction appears partially aligned.

\subsection{Shift Magnitude Ratios}

For each unlearning method, we compute the relative shift magnitude
$R^S =
\frac{\|\Delta_{\theta_o\rightarrow\theta_u}^S\|}{\|\Delta_{\theta_o\rightarrow\theta_r}^S\|},
\qquad S\in\{\mathcal{D}_u,\mathcal{D}_r\},$
separately for forget samples $\mathcal{D}_u$ and retain samples $\mathcal{D}_r$. A value of $R^S=1$ indicates that unlearning matches retraining in shift magnitude, $R^S>1$ indicates over-adjustment, and $R^S<1$ indicates under-adjustment.

\begin{table}[ht]
\centering
\small
\setlength{\tabcolsep}{4pt}
\renewcommand{\arraystretch}{1.05}
\begin{tabular}{l|cc|cc}
\toprule
\multirow{2}{*}{\textbf{Method}} 
& \multicolumn{2}{c|}{\textbf{Forget Samples } ($\mathcal{D}_u$)} 
& \multicolumn{2}{c}{\textbf{Retain Samples } ($\mathcal{D}_r$)} \\
\cmidrule(lr){2-3} \cmidrule(lr){4-5}
& $\cos(\Delta_u)\uparrow$ & $R^u$ 
& $\cos(\Delta_r)\uparrow$ & $R^r$ \\
\midrule
SCRUB~\cite{kurmanji2023towards}         & 0.945 & 1.074 & -0.345 & 4.578 \\
Boundary Shrink~\cite{chen2023boundary}  & 0.916 & 1.170 & -0.305 & 4.606 \\
UNSIR~\cite{tarun2023fast}              & 0.720 & 1.741 &  0.206 & 7.282 \\
Amnesiac~\cite{graves2021amnesiac}       & 0.795 & 1.612 &  0.432 & 1.776 \\
SSD~\cite{foster2024fast}                 & 0.911 & 1.312 & -0.389 & 2.608 \\
POUR-P~\cite{le2025pour}                 & 0.854 & 1.092 &  0.468 & 0.761 \\
POUR-D~\cite{le2025pour}                 & 0.884 & 1.097 &  0.074 & 1.191 \\
\bottomrule
\end{tabular}
\caption{Directional alignment and relative shift magnitude of unlearning with respect to retraining on CIFAR-10 with ResNet-18. Here $R^u$ and $R^r$ denote the ratio of unlearning shift magnitude to retraining shift magnitude on forget and retain samples, respectively.}
\label{tab:magnitude_diagnostic}
\end{table}

Table~\ref{tab:magnitude_diagnostic} refines the main paper's directional diagnosis. On forget samples $\mathcal{D}_u$, methods such as SCRUB, Boundary Shrink, SSD, and the POUR variants show high cosine similarity together with magnitude ratios close to $1$, indicating that their forget-side shifts are at least roughly similar to retraining in both direction and scale. In contrast, UNSIR and Amnesiac show weaker directional agreement together with substantially larger magnitude ratios, suggesting that they over-adjust even on the forget set.

On retain samples $\mathcal{D}_r$, the picture changes sharply. Most methods have magnitude ratios above $1$, often far above $1$, while cosine similarity is near zero or even negative. This indicates that retain-side behavior is not merely misdirected relative to retraining, but frequently over-displaced as well. In other words, the retain-side mismatch identified in the main paper reflects both directional inconsistency and excessive adjustment strength.

As a summary, these magnitude comparisons support the same qualitative conclusion as the main paper. The asymmetry between forget and retain samples is not only directional. It also appears in the scale of the representation shift, reinforcing the view that current unlearning methods only partially approximate retraining and often do so unevenly across the forget and retain sets.

\section{Extended Projection / Subspace Analysis}
\label{appendix:subspace_analysis}

The main paper shows that representation mismatch is concentrated along the retraining shift direction rather than being diffusely distributed across representation space. This appendix provides detailed interpretations against the construction controls as well as additional views of the same projection decomposition. The goal is not to introduce a new analysis, but to verify that the main structured-residual diagnosis preserves its directional-specificity under valid controls and remains visible under complementary summaries of the parallel and orthogonal components.

\begin{figure}[t]
    \centering
    \begin{minipage}{\textwidth}
        \centering
        \begin{subfigure}{1\textwidth}
            \centering
            \includegraphics[width=\textwidth]{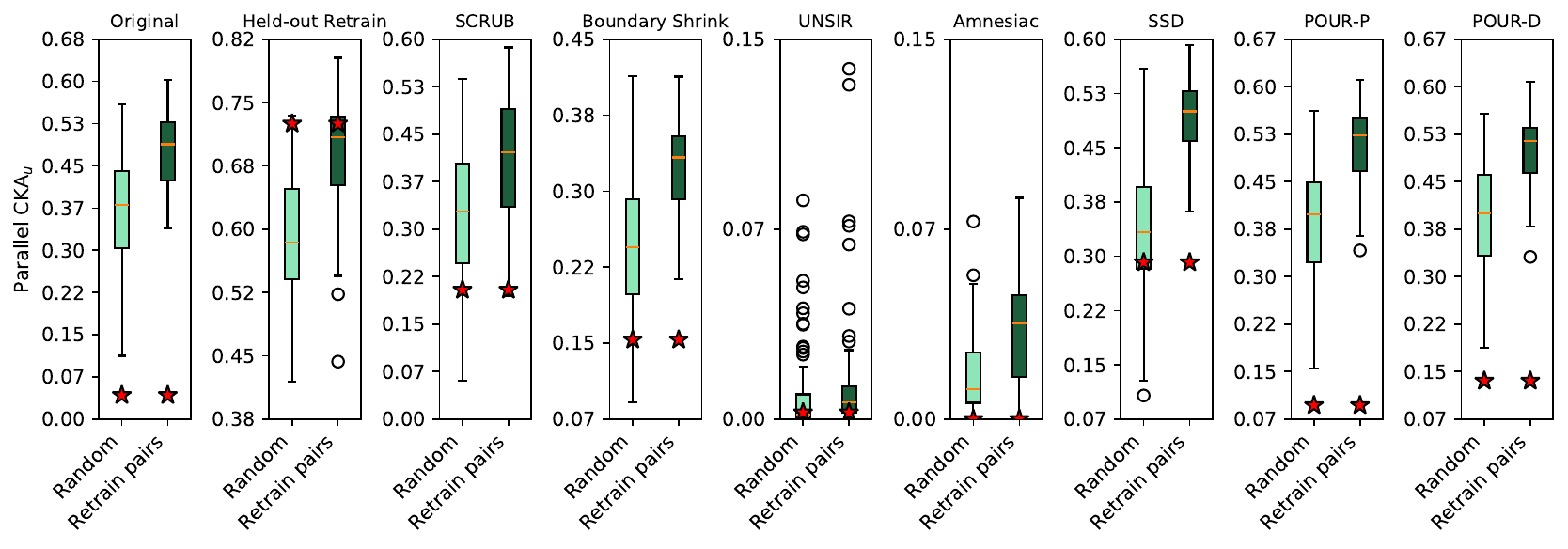}
            \caption{Parallel $\mathrm{CKA}_u$ against random and retrain pairs control distributions. Lower than the distribution indicates a strong directional discrepancy along $v^{D_u}$ relative to the control.}
            \label{fig:parallel_separate}
        \end{subfigure}
        \hfill
        \begin{subfigure}{1\textwidth}
            \centering
            \includegraphics[width=\textwidth]{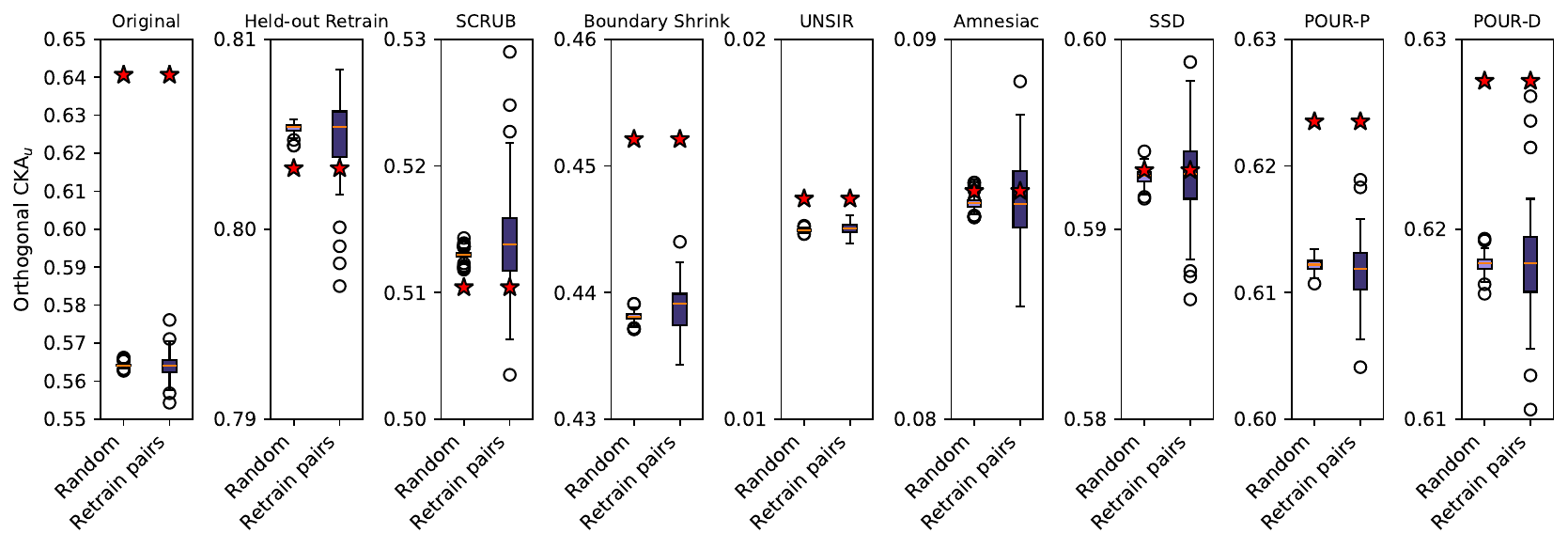}
            \caption{Orthogonal $\mathrm{CKA}_u$ against random and retrain pairs control distributions. Higher than or within the distribution indicates relatively well-preserved or consistent representation structure in the orthogonal subspace to the control.}
            \label{fig:orthogonal}
        \end{subfigure}
        \caption{Matched-dimensional comparison of decomposed $\mathrm{CKA}_u$ against construction control. Several unlearned methods show a tendency towards directionally organized representation mismatch along $v_{D_u}$, particularly relative to the variability observed across independent retrained models, whereas some exhibit broader representation disruption.}
        \label{fig:parallel_orthogonal_control}
    \end{minipage}
\end{figure}

\subsection{Retraining Shift Direction Construction Control}

The localized representation mismatch along $v^{D_u}$ compared to its orthogonal components alone does not provide enough evidence that it is empirically significant. To test the directional-specificity demonstrated by lower $\mathrm{CKA}_u$, we conduct two control experiments using random directions and retrained-to-retrained directions.

As the parallel component is one-dimensional whereas the orthogonal component spans an $(d-1)$-dimensional subspace, it is possible that CKA between representations projected onto an arbitrary one-dimensional direction may also be similarly low. To rule out this explanation, for each method, we compute $\mathrm{CKA}_u$ using representations projected along 100 random directions and their corresponding orthogonal components. We then compare the parallel and orthogonal $\mathrm{CKA}_u$ along $v^{D_u}$ against their matched-dimensional random direction distributions. If the parallel $\mathrm{CKA}_u$ is unusually low relative to the random control, and the orthogonal $\mathrm{CKA}_u$ remains similar to or higher than its corresponding control, this would establish that the discrepancy is directionally specific along $v^{D_u}$ rather than being uniformly distributed.

More importantly, to align with our retraining-consistent protocol, we need to rule out the explanation that the low CKA phenomenon is reproducible under the natural representation shift produced by stochastic retraining. Using the 10 independent retrained models, we construct 45 non-identical forget set representation shift directions from retrained model $i$ to retrained model $j$:

\begin{equation}
v_{ij}^{D_u} = \frac{\Delta_{\theta_{r_i}\rightarrow\theta_{r_j}}^{D_u}}{\|\Delta_{\theta_{r_i}\rightarrow\theta_{r_j}}^{D_u}\|}
\end{equation}

We then compute the $\mathrm{CKA}_u$ on the representations projected along $v_{ij}^{D_u}$ and their corresponding orthogonal components, forming empirical retrain pairs null direction distributions. If the parallel $\mathrm{CKA}_u$ is lower than the retrain pairs control while orthogonal $\mathrm{CKA}_u$ is similar to or higher than its control, the discrepancy along $v^{D_u}$ exceeds the ordinary retraining variability.

From \fig~\ref{fig:parallel_orthogonal_control}, the parallel $\mathrm{CKA}_u$ is lower than the orthogonal $\mathrm{CKA}_u$ for all methods. We therefore examine the controls for the parallel and orthogonal components separately in \fig~\ref{fig:parallel_separate} and \fig~\ref{fig:orthogonal}. As a positive control for retraining, $\mathrm{CKA}_u$ is also computed between two held-out retrained models that were not used to construct the 45 retrain pairs directions. Both its parallel and orthogonal $\mathrm{CKA}_u$ fall within the retrain pairs controls, while lying toward the upper tail of their random controls. This demonstrates that the control procedure does not inherently produce low CKA along retraining-derived directions.

The Original model, POUR-P, and POUR-D exhibit a strong directional discrepancy: their parallel $\mathrm{CKA}_u$ falls below both the random and retrain pairs distributions, while their orthogonal $\mathrm{CKA}_u$ exceeds both control distributions. These results provide evidence that the representation discrepancy is preferentially localized along the retraining-related direction, exceeding both arbitrary directional variation and retraining variation, while the orthogonal representation structure remains relatively well preserved.

Meanwhile, Boundary Shrink shows a stronger directional discrepancy relative to the retrain pairs control, but only a moderate tendency toward low parallel similarity relative to random control, with parallel $\mathrm{CKA}_u$ falling below its 25th percentile. The orthogonal $\mathrm{CKA}_u$ lies above both controls, indicating a relatively well-preserved orthogonal representation structure. Although the directional effect is only moderate relative to arbitrary directions, the discrepancy is consistently larger than that observed across retrainings, indicating a structured mismatch relative to the retraining reference. 

SSD shows a strong discrepancy relative to the retrain pairs control, but a weak tendency toward low parallel similarity relative to random control, with its parallel $\mathrm{CKA}_u$ near the median. Its orthogonal $\mathrm{CKA}_u$ remains within both the control distributions, indicating no unusual discrepancy in the orthogonal subspace. Thus, SSD exhibits a representation discrepancy that is more pronounced relative to retraining variability along the retraining-related direction, despite limited evidence of directional specificity relative to arbitrary directions.

SCRUB has a moderate tendency towards low parallel mismatch relative to the retrain pairs control, with parallel $\mathrm{CKA}_u$ near the lower tail while its orthogonal $\mathrm{CKA}_u$ remains within the distribution. Relative to the random control, the parallel $\mathrm{CKA}_u$ falls below the 25th percentile, whereas the orthogonal $\mathrm{CKA}_u$ falls below the lower tail. Thus, the discrepancy in SCRUB is moderately directional with no unusual mismatch in orthogonal space relative to retraining variability, but shows no directional evidence relative to arbitrary directions.

Both UNSIR and Amnesiac exhibit severe representation mismatch across both parallel and orthogonal components, with $\mathrm{CKA}_u$ values less than 0.1. Within this broader mismatch, Amnesiac shows particularly low parallel similarity relative to its controls, whereas the orthogonal component of UNSIR is relatively more preserved compared with its controls. However, both methods remain characterized by substantial mismatch in both components.

To sum up the representation discrepancy structure across methods, several unlearned methods show a directionally preferential mismatch along $v_{D_u}$, particularly relative to the variability observed across independent retrained models, whereas others exhibit broader representation disruption.

\begin{figure}[t]
    \centering
    \begin{minipage}{0.30\textwidth}
        \centering
        \includegraphics[width=\textwidth]{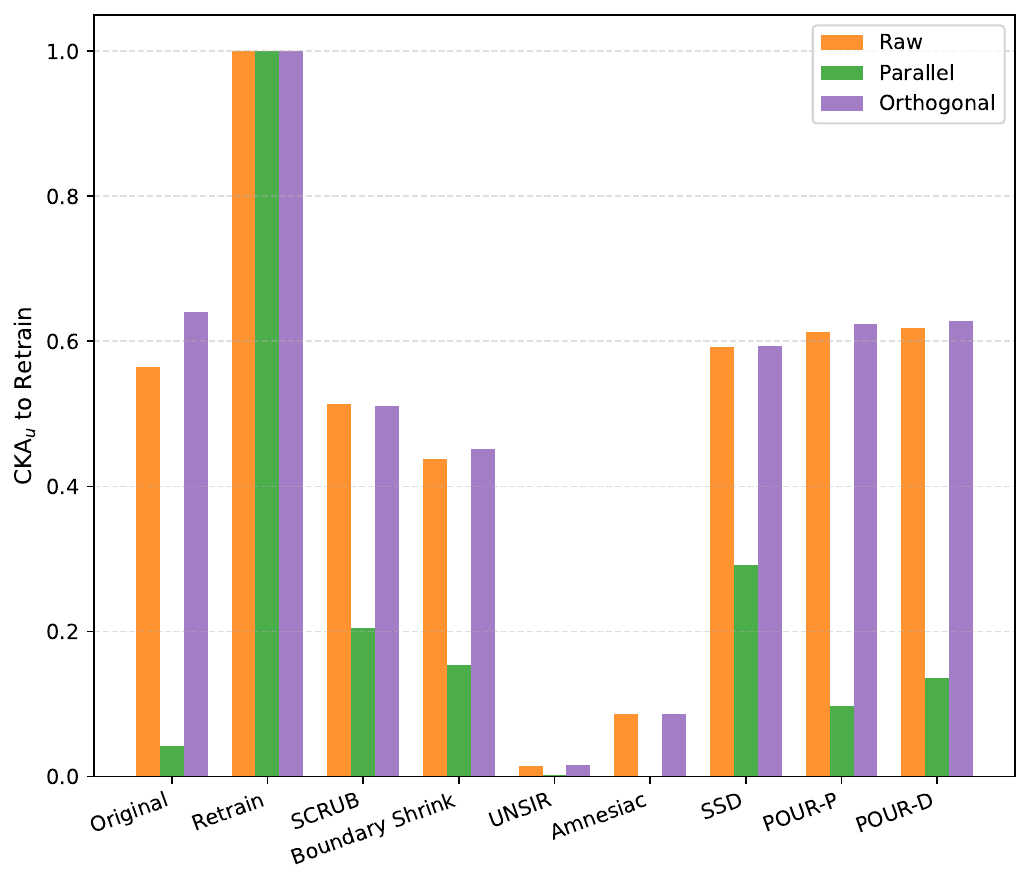}
        \caption{Residual discrepancy of the raw representation together with its parallel and orthogonal components.}
        \label{fig:ckau_raw_subspace}
    \end{minipage}
    \hfill
    \nextfloat
    \begin{minipage}{0.30\textwidth}
        \centering
        \includegraphics[width=\textwidth]{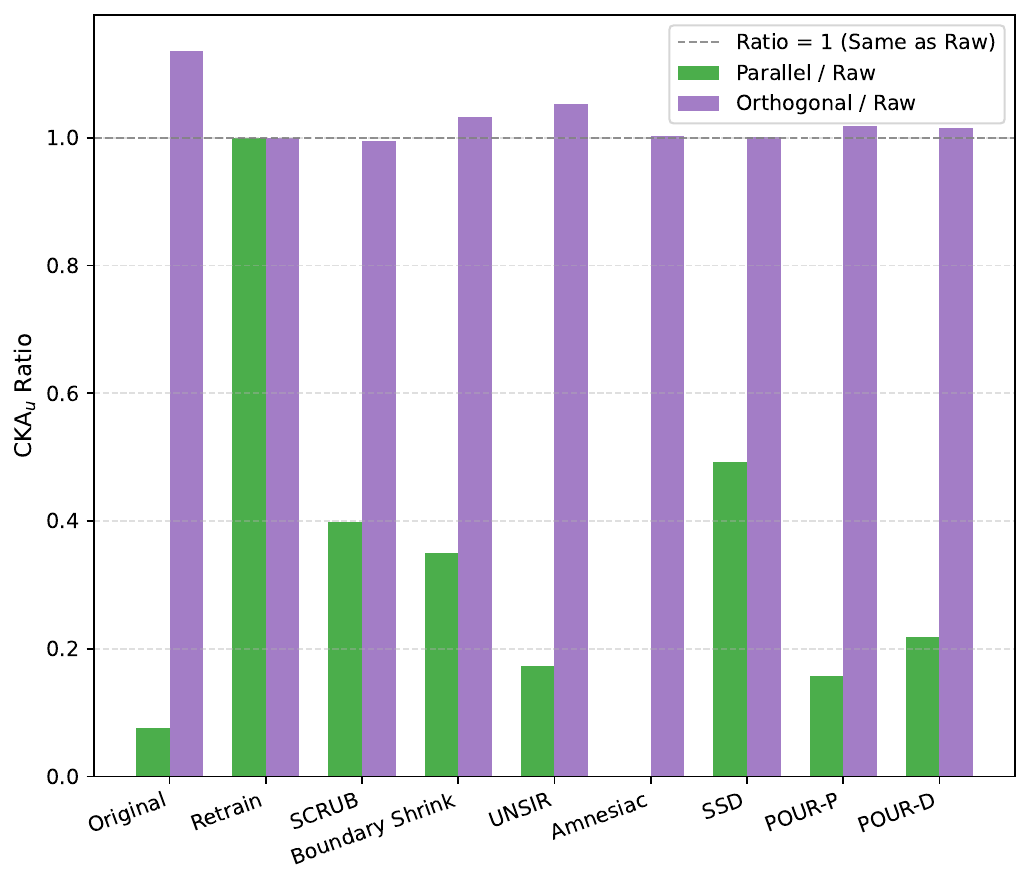}
        \caption{Residual discrepancy ratio relative to the raw representation.}
        \label{fig:cka_f_ratio_subspace}
    \end{minipage}
    \hfill
    \nextfloat
    \begin{minipage}{0.30\textwidth}
        \centering
        \includegraphics[width=\textwidth]{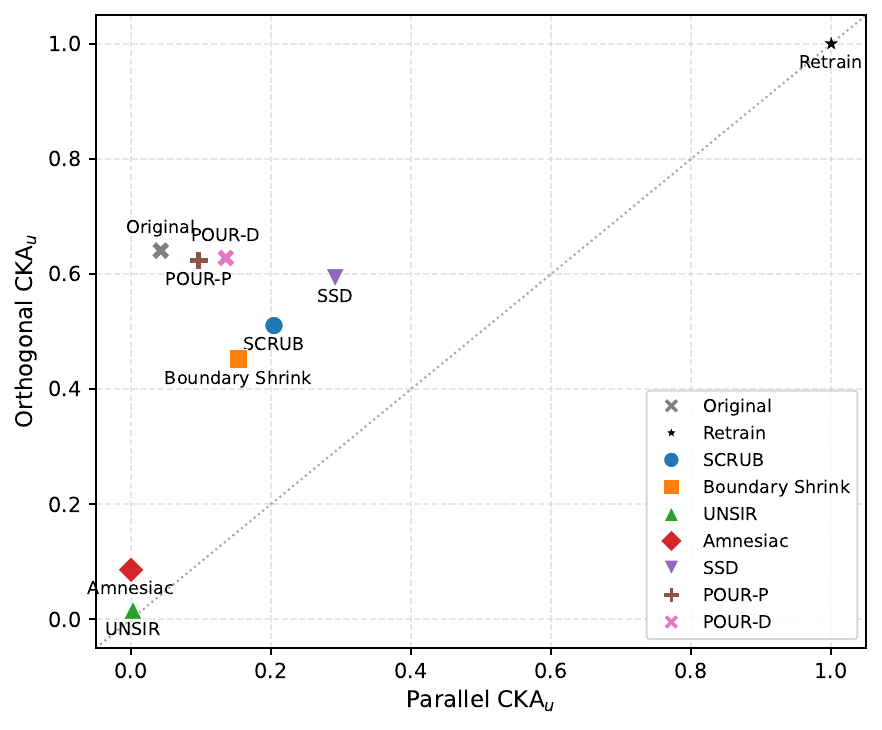}
        \caption{Scatter plots of residual discrepancy across the parallel and orthogonal components.}
        \label{fig:cka_f_scatter_subspace}
    \end{minipage}
    \vspace{-10pt}
\end{figure}

\subsection{Additional Projection Decomposition Views}
\paragraph{Raw, Parallel, and Orthogonal Comparisons}

\fig~\ref{fig:ckau_raw_subspace} compares the raw representation with its parallel and orthogonal components for $\mathrm{CKA}_u$. Lower $\mathrm{CKA}_u$ indicate larger discrepancy relative to retraining. The same qualitative pattern appears across methods. The orthogonal component typically remains close to the raw representation, whereas the parallel component shows stronger deviation. In the original model, the parallel component already carries a substantially larger discrepancy than the orthogonal component. After unlearning, this imbalance persists. Although the absolute discrepancy often decreases relative to the original model, the parallel component remains the more informative direction of residual mismatch.

\paragraph{Parallel-to-Raw and Orthogonal-to-Raw Ratios}

To quantify how much of the residual discrepancy lies in each component, we normalize each subspace metric by its raw counterpart. For $\mathrm{CKA}_u$, we compute
\[
r_{\parallel}=\frac{\mathrm{parallel}}{\mathrm{raw}},
\qquad
r_{\perp}=\frac{\mathrm{orthogonal}}{\mathrm{raw}}.
\]
A ratio near $1$ indicates that the subspace carries discrepancy comparable to the raw representation. For $\mathrm{CKA}_u$, values below $1$ indicate stronger mismatch than the raw representation. \fig~\ref{fig:cka_f_ratio_subspace} shows a stable asymmetry across methods. The orthogonal ratios remain close to $1$, indicating that the orthogonal component largely tracks the raw discrepancy. In contrast, the parallel ratios consistently point to stronger discrepancy, with $r_{\parallel}$ typically falling below $1$ for $\mathrm{CKA}_u$.

\paragraph{Additional Scatter Plots}

\fig~\ref{fig:cka_f_scatter_subspace} provides a final visual summary of the same decomposition. For $\mathrm{CKA}_u$, we plot ${\mathrm{CKA}_u}_{\parallel}$ against ${\mathrm{CKA}_u}_{\perp}$. The diagonal corresponds to equal discrepancy in the two components. The points systematically fall away from it. For $\mathrm{CKA}_u$, most methods occupy the complementary regime in which the parallel component exhibits lower similarity to retraining than the orthogonal component.

\section{Extended Scaling Summaries}
\label{appendix:extended_scaling}

The main paper argues that the diagnosed mismatch is not a narrow artifact of one benchmark configuration, but a stable property of current unlearning behavior. This appendix provides additional output-level and directional summaries across dataset complexity, model size, and architecture. The purpose is not to expand the paper into a benchmark comparison, but to show that the main persistence claim continues to hold across a broader range of experiments.

\subsection{Additional Output-level Accuracy Tables}

We first report output-level accuracy across four additional settings to rule out the weaker explanation that the representation-level failures identified in the main paper are simply caused by poor output-level forgetting. For each setting, we report forget-set accuracy $\mathrm{Acc}_u$, retain-set accuracy $\mathrm{Acc}_r$, and the harmonic mean
$\mathrm{HM}
=
\frac{2\cdot \mathrm{Acc}_r \cdot (100-\mathrm{Acc}_u)}
{\mathrm{Acc}_r + 100-\mathrm{Acc}_u},$
with the original and retrained models serving as lower and upper references, respectively. Bold indicates the top three $\mathrm{HM}$ values among unlearning methods. All metrics are reported in $\%$.

\begin{figure}[t]
    \centering
    \begin{minipage}{0.45\textwidth}
        \centering
        \captionsetup{type=table}
        {\small
        \begin{tabular}{l|ccc}
        \toprule
        \textbf{Method} & $\mathrm{Acc}_u$ $\downarrow$ & $\mathrm{Acc}_r$ $\uparrow$ & $\mathrm{HM}$ $\uparrow$ \\
        \midrule
        Original & 99.41 & 96.87 & 1.17 \\
        Retrain  & 0.00  & 95.91 & 97.91 \\
        \midrule
        SCRUB     & 0.39  & 99.20 & \textbf{99.40} \\
        Boundary Shrink  & 1.41  & 49.77 & 66.15 \\
        UNSIR     & 0.00  & 21.33 & 35.16 \\
        Amnesiac  & 0.00  & 91.18 & \textbf{95.39} \\
        SSD       & 0.00  & 64.11 & 78.13 \\
        POUR-P    & 0.00  & 96.87 & \textbf{98.41} \\
        POUR-D    & 0.00  & 65.57 & 79.21 \\
        \bottomrule
        \end{tabular}
        }
        \caption{Output-level accuracy for CIFAR-100 / ResNet-18.\\[\baselineskip]}
        \label{tab:cls_cifar100_r18}
    \end{minipage}
    \hfill
    \begin{minipage}{0.45\textwidth}
        \centering
        \captionsetup{type=table}
        {\small
        \begin{tabular}{l|ccc}
        \toprule
        \textbf{Method} & $\mathrm{Acc}_u$ $\downarrow$ & $\mathrm{Acc}_r$ $\uparrow$ & $\mathrm{HM}$ $\uparrow$ \\
        \midrule
        Original & 96.79 & 91.99 & 6.20 \\
        Retrain  & 0.00  & 93.57 & 96.68 \\
        \midrule
        SCRUB & 67.77 & 98.28 & 48.54 \\
        Boundary Shrink  & 0.22  & 44.55 & 61.60 \\
        UNSIR     & 0.00  & 19.69 & 32.90 \\
        Amnesiac  & 0.00  & 90.85 & \textbf{95.21} \\
        SSD       & 0.00  & 68.62 & 81.39 \\
        POUR-P    & 0.00  & 91.65 & \textbf{95.64} \\
        POUR-D    & 0.00  & 78.88 & \textbf{88.19} \\
        \bottomrule
        \end{tabular}
        }
        \caption{Output-level accuracy for CIFAR-100 / ResNet-50. SCRUB fails to complete forgetting, with $\mathrm{Acc}_u = 67.77\%$ and $\mathrm{HM} = 48.54\%$.}
        \label{tab:cls_cifar100_r50}
    \end{minipage}
\end{figure}
\begin{figure}[t]
    \centering
    \begin{minipage}{0.45\textwidth}
        \centering
        \captionsetup{type=table}
        {\small
        \begin{tabular}{l|ccc}
        \toprule
        \textbf{Method} & $\mathrm{Acc}_u$ $\downarrow$ & $\mathrm{Acc}_r$ $\uparrow$ & $\mathrm{HM}$ $\uparrow$ \\
        \midrule
        Original & 92.22 & 74.63 & 14.09 \\
        Retrain  & 0.00  & 78.07 & 87.68 \\
        \midrule
        SCRUB     & 26.50 & 79.07 & \textbf{76.18} \\
        Boundary Shrink  & 10.04 & 58.23 & 70.70 \\
        UNSIR     & 0.00  & 53.44 & 69.66 \\
        Amnesiac  & 1.00  & 76.48 & \textbf{86.29} \\
        SSD       & 0.00  & 22.91 & 37.28 \\
        POUR-P    & 0.00  & 74.60 & \textbf{85.45} \\
        POUR-D    & 0.40  & 63.07 & 77.23 \\
        \bottomrule
        \end{tabular}
        }
        \caption{Output-level accuracy for TinyImageNet / ResNet-18.}
        \label{tab:cls_tinyimagenet_r18}
    \end{minipage}
    \hfill
    \begin{minipage}{0.45\textwidth}
        \centering
        \captionsetup{type=table}
        {\small
        \begin{tabular}{l|ccc}
        \toprule
        \textbf{Method} & $\mathrm{Acc}_u$ $\downarrow$ & $\mathrm{Acc}_r$ $\uparrow$ & $\mathrm{HM}$ $\uparrow$ \\
        \midrule
        Original & 99.80 & 99.95 & 0.39 \\
        Retrain  & 0.00  & 99.95 & 99.97 \\
        \midrule
        SCRUB     & 0.00  & 99.94 & \textbf{99.97} \\
        Boundary Shrink  & 9.81  & 87.44 & 88.79 \\
        UNSIR     & 0.00  & 39.92 & 57.06 \\
        Amnesiac  & 0.00  & 84.15 & 91.39 \\
        SSD       & 0.00  & 80.54 & 89.22 \\
        POUR-P    & 0.00  & 99.95 & \textbf{99.98} \\
        POUR-D    & 0.21  & 98.77 & \textbf{99.28} \\
        \bottomrule
        \end{tabular}
        }
        \caption{Output-level accuracy for CIFAR-100 / ViT-Tiny.}
        \label{tab:cls_cifar100_vit}
    \end{minipage}
\end{figure}

Taken together, these tables support the main paper's interpretation rather than weaken it. In most settings, several methods still achieve strong output-level forgetting and high harmonic mean, so the representation-level failures discussed in the paper cannot be dismissed as a trivial byproduct of poor output behavior. This is especially true for methods such as Amnesiac and POUR-P, which often remain strong under output-level evaluation while still exhibiting directional and geometric mismatch relative to retraining. At the same time, scaling makes some failures more visible, with SCRUB no longer forgetting effectively on CIFAR-100 with ResNet-50 (\tab~\ref{tab:cls_cifar100_r50}) and on TinyImageNet with ResNet-18 (\tab~\ref{tab:cls_tinyimagenet_r18}), while Boundary Shrink and UNSIR continue to suffer severe retain-side degradation (\tab~\ref{tab:cls_cifar100_r18},~\ref{tab:cls_cifar100_r50},~\ref{tab:cls_cifar100_vit}) . Thus, broader output-level summaries do not overturn the main diagnosis. They reinforce the central claim that strong output-level behavior can coexist with retraining-inconsistent representation-level structure.

\subsection{Additional Directional Alignment Plots}

We next report directional alignment across the same four settings to test whether the forget/retain asymmetry identified in the main paper persists beyond the CIFAR-10 / ResNet-18 baseline. For each method, we plot the cosine similarity between the unlearning and retraining shifts, computed separately for forget and retain samples. Values closer to $1$ indicate stronger directional agreement with retraining, values near $0$ indicate orthogonality, and negative values indicate opposing directions.

\begin{figure}[t]
    \centering
    \begin{minipage}[t]{0.48\textwidth}
        \centering
        \includegraphics[width=\textwidth]{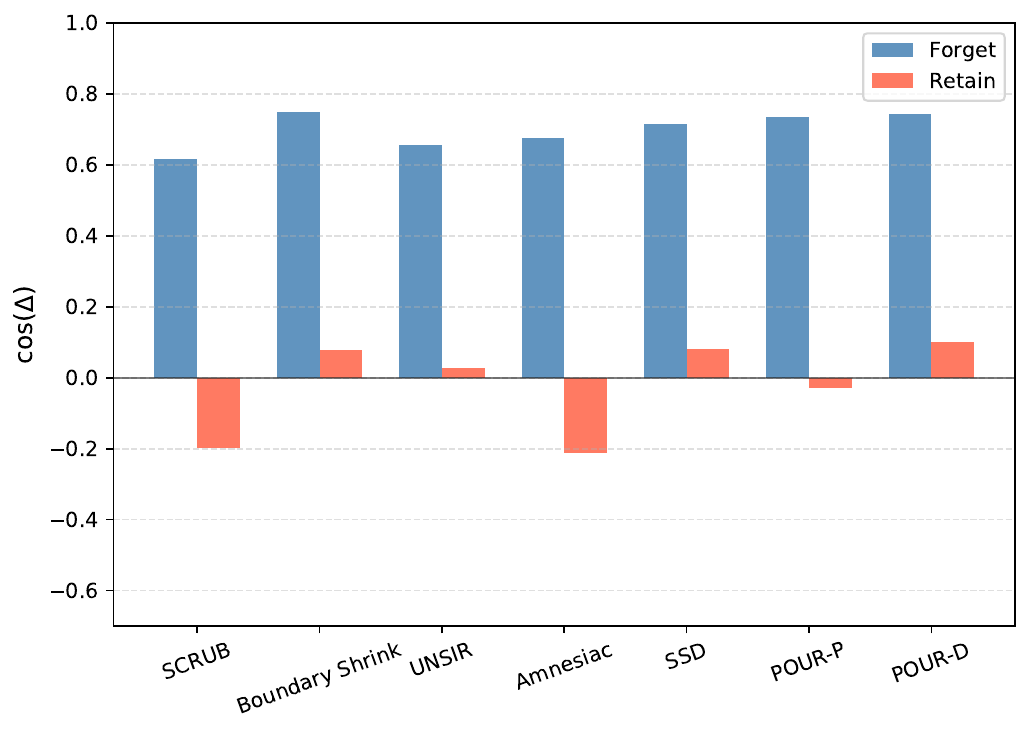}
        \captionsetup{type=figure}
        \caption{Directional alignment for CIFAR-100 / ResNet-18.}
        \label{fig:dir_cf100_rn18}
    \end{minipage}
    \hfill
    \begin{minipage}[t]{0.48\textwidth}
        \centering
        \includegraphics[width=\textwidth]{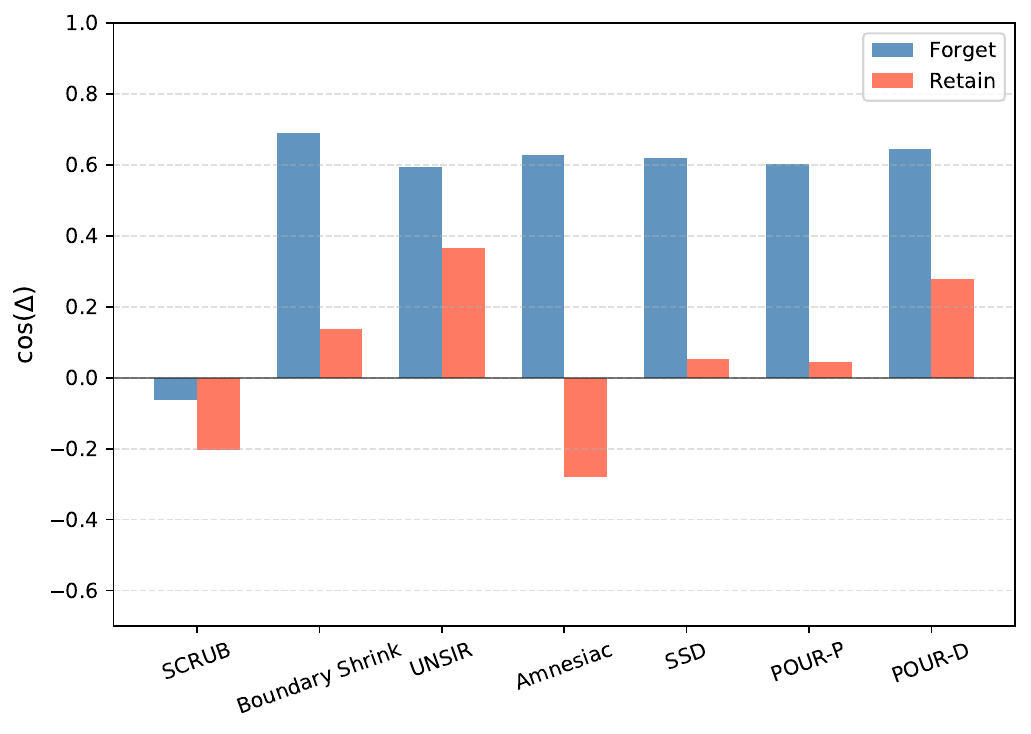}
        \captionsetup{type=figure}
        \caption{Directional alignment for CIFAR-100 / ResNet-50. SCRUB shows negative alignment on both forget and retain samples, indicating that its representation shift moves in the opposite direction from retraining.}
        \label{fig:dir_cf100_rn50}
    \end{minipage}
\end{figure}
\begin{figure}[t]
    \centering
    \begin{minipage}[t]{0.48\textwidth}
        \centering
        \includegraphics[width=\textwidth]{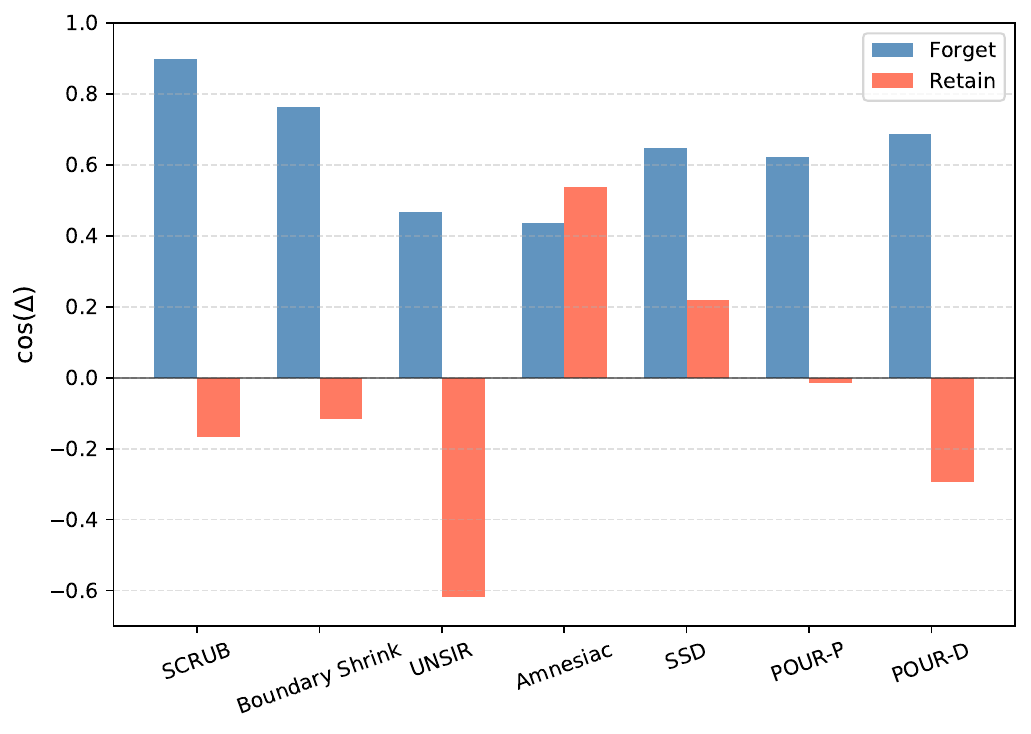}
        \captionsetup{type=figure}
        \caption{Directional alignment for TinyImageNet / ResNet-18. Amnesiac shows slightly higher alignment on retain samples than on forget samples in this setting, deviating from the more common pattern observed elsewhere.}
        \label{fig:dir_tin_rn18}
    \end{minipage}
    \hfill
    \begin{minipage}[t]{0.48\textwidth}
        \centering
        \includegraphics[width=\textwidth]{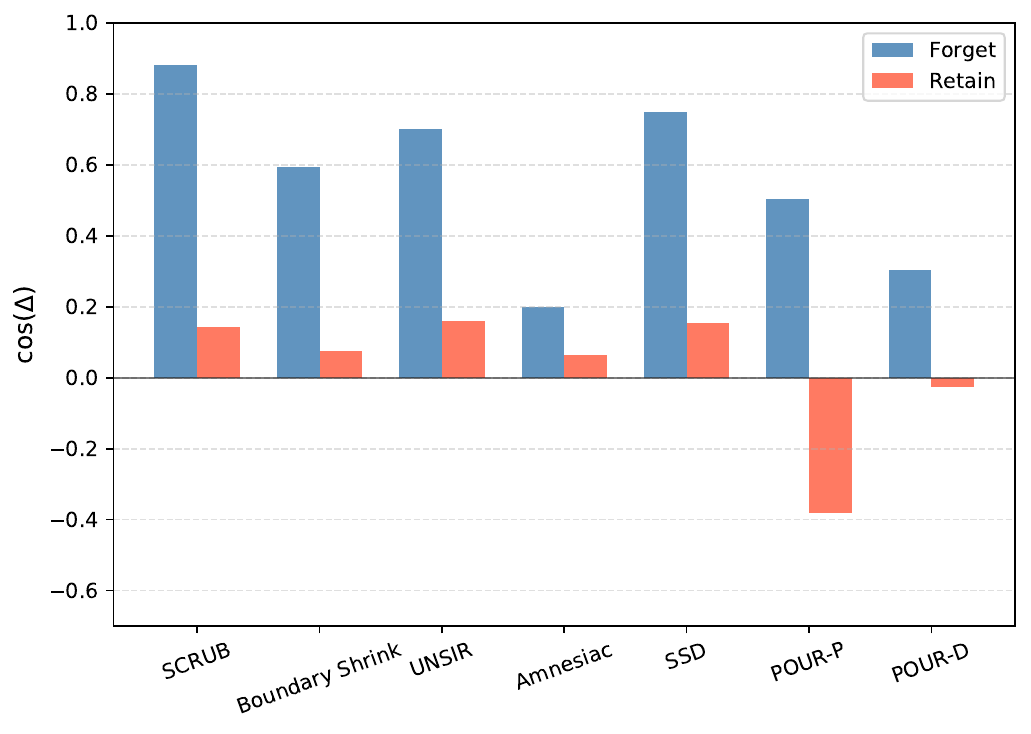}
        \captionsetup{type=figure}
        \caption{Directional alignment for CIFAR-100 / ViT-Tiny.}
        \label{fig:dir_cf100_vit}
    \end{minipage}
\end{figure}

Across these settings, the same qualitative asymmetry persists. Forget-side alignment is typically moderate to high, whereas retain-side alignment often remains near zero or below, especially in \fig~\ref{fig:dir_cf100_rn18} and \ref{fig:dir_cf100_vit}. This shows that the main paper's directional diagnosis is not tied to a single benchmark or backbone. The two clearest exceptions, namely Amnesiac on TinyImageNet (\fig~\ref{fig:dir_tin_rn18}) and SCRUB on CIFAR-100 with ResNet-50 (\fig~\ref{fig:dir_cf100_rn50}), do not weaken the argument; they represent stricter failures in which the method does not align with retraining even on forget samples. Overall, these additional plots reinforce the main paper's conclusion that current unlearning methods may partially align forget representations with retraining, yet remain persistently inconsistent with retain representations across increased dataset complexity, model size, and architectural change.

\section{Extended Discussion}
\label{appendix:extended_failure_mode}

This appendix expands the discussion in Section~\ref{sec:failure_mode}. It provides a fuller interpretation of the diagnosed failure mode, together with its implications for evaluation and method design.

\subsection{A Hidden Failure Mode of Current Machine Unlearning}

The central result of this paper is that current machine unlearning suffers from a hidden failure mode that standard output-level evaluation systematically fails to detect.

First, standard output-level success signals are too weak. Low forget-set accuracy, low logit-level membership inference, and strong retain accuracy can make an unlearned model appear successful even when substantial forget information remains recoverable in representation space. In this sense, current evaluation can systematically mistake \emph{apparent forgetting} for successful forgetting.

Second, this failure is \emph{structured rather than random}. Current methods are not simply wrong in arbitrary ways. Instead, a retraining-consistent representation lens reveals a repeated pattern in which these methods often partially align with retraining on forget samples, diverge on retain samples, and leave residual mismatch that is structured rather than diffuse. Moreover, these failure modes survive the variability of independent retraining and pose an operational consequence through relearning attacks. The residual mismatch is therefore not well described as generic optimization error; it reflects a systematic deviation from retraining-consistent forgetting.

Taken together, these findings suggest that many current unlearning methods remain substantially inconsistent with retraining in representation space, even when they appear successful under output-level evaluation. In this sense, output-level forgetting alone can be a misleading indicator of successful unlearning under a stronger representation-level lens.

\subsection{Implications for Evaluation}

The most immediate implication of this failure mode is that current evaluation practice is incomplete. If output-level metrics alone can be satisfied while representation-level residuals remain recoverable, then prevailing evaluation protocols can systematically overestimate successful unlearning. This does not mean that output-level metrics are uninformative. They remain useful for measuring prediction-layer behavior and black-box leakage. However, they are not strong enough to serve as the sole success signal when the goal is to remove the influence of the forget set from the model.

Our results therefore support a stricter evaluation principle in which \unlearning~is assessed relative to a retrained reference rather than by output behavior alone. In particular, evaluation should ask whether the unlearned model is \emph{retraining-consistent} in representation space, and whether the residual discrepancy is random or structured. Under this lens, the key question is no longer just whether the model looks forgotten at the output layer, but whether it changes in a way that resembles retraining without the forget data.

\subsection{Implications for Method Design}

The same failure mode also has implications for algorithm design. Many current methods are optimized around output-level objectives, such as reducing forget-set confidence or weakening output-level membership signals. Our results suggest that this is not sufficient to ensure consistency with retraining in representation space. A method can satisfy these output-level objectives while still exhibiting asymmetric and structured retraining-inconsistent residuals.

This does not imply that current methods are useless, nor that all of them fail in the same way. Rather, it suggests that many methods are better understood as achieving \emph{apparent forgetting} rather than retraining-consistent forgetting. Future unlearning methods may therefore benefit from stronger representation-level principles, including reducing structured residuals in representation space and better approximating retraining-consistent transformations.

\subsection{Broader Perspective}

The broader message of this paper is that the most informative structure in unlearning is not fully visible at the output layer. Output-level metrics remain useful, but they can be too weak to reveal whether unlearning actually follows retraining in representation space. In this sense, retraining reveals what output forgetting hides by showing not only that residual discrepancy remains, but also how that discrepancy is organized.

Seen this way, the central issue is not merely that some metrics are incomplete, but that current evaluation practice can certify successful-looking unlearning on the basis of a signal that is too weak to rule out structured retraining-inconsistent residuals.

\section{Limitations and Scope}
\label{appendix:alternative_views}

Our goal is not to argue that output-level metrics or retraining-based references are useless in all settings. Rather, the claim of this paper is narrower and more specific. By themselves, they are \emph{insufficient} to rule out a hidden failure mode in which apparent output-level forgetting coexists with retraining-inconsistent residuals in representation space.

\paragraph{Output-level metrics may be sufficient in practice.}
A natural counterargument is that many realistic adversaries only have black-box access to model outputs, so evaluating \unlearning~through output-level quantities such as $\mathrm{MIA}_{\mathrm{logit}}$ may be practically sufficient~\cite{kurmanji2023towards,chundawat2023can}. We agree that output-level metrics remain useful in such settings. However, this does not alter the paper's central diagnosis. A growing number of models are released with open weights or otherwise permit access to internal representations~\cite{touvron2023llama}, and white-box access is also natural in auditing and compliance settings for data deletion~\cite{sun2026statistical}. In these cases, output-level success alone is insufficient to rule out structured residual discrepancies in the representation space. Our results show that a model can appear forgotten at the output layer while still remaining substantially inconsistent with retraining internally. Thus, output-level evaluation remains informative, but it is not sufficient as a general success signal for the stronger notion of forgetting studied here.

\paragraph{Retraining may be too expensive to serve as a practical reference.}
A second objection is that retraining from scratch is often computationally infeasible in large-scale settings~\cite{spartalis2025lotus,yu2025impossibility}, making $\theta_r$ an impractical operational target. We agree that retraining is often too costly to use as a deployment-time procedure. But this is different from the role it plays in our paper. Here, $\theta_r$ is used as a \emph{benchmark-scale reference} for what stronger forgetting should look like, not as a claim that every practical system must retrain in deployment. Specifically, in the primary setting CIFAR-10/ResNet-18, ten reference retrains required 3,519 seconds in total, equivalent to 0.98 A100 GPU-hours. This is a one-time benchmark cost for a fixed dataset, architecture, training recipe, preprocessing pipeline, forget-set configuration, and diagnostic implementation. The same reference population is reused across all seven evaluated methods and all diagnostics for the fixed experimental configuration.

As the goal of \unlearning~is to approximate the model that would have been obtained without the influence of the forget set, retraining provides a natural comparative reference for diagnosing failure. In our controlled setting, it captures this ideal model and therefore reveals mismatches that weaker endpoint-based criteria may miss. We do not argue that retraining-consistent geometry is the only valid notion of forgetting, instead, our use of retraining is comparative rather than absolute, highlighting that output-level criteria may certify apparent forgetting even when substantial retraining-inconsistent residuals remain detectable in representation space.

\paragraph{Scalability to foundational model and large-scale dataset.}

Our one-time benchmark cost demonstrates feasibility only in the primary CIFAR-10/ResNet-18 setting. Extending the protocol to large-scale datasets such as full ImageNet would require a new population of retrained references under the same large-scale training configuration. For foundation models, retraining from scratch may be prohibitively expensive. When deletion is limited to a downstream adaptation dataset, the relevant reference could instead be obtained by retraining the adaptation from the released base checkpoint while excluding the designated examples. However, this does not address deletion from the pretraining data. We therefore position our protocol as a benchmark-time research and auditing diagnostic, rather than a deployment-time verifier. The absence of full-ImageNet, LLM, or diffusion-model validation limits the scope of our empirical evaluation, but does not change the conclusions established across CIFAR-10, CIFAR-100, TinyImageNet, ResNet-18/50, and ViT-Tiny.

\paragraph{Generalization to Instance-level Unlearning.}
Instance-level and concept-level unlearning likewise require setting-appropriate references and diagnostics, which are not covered by our current empirical evaluation. Our validated findings focus on class-level unlearning in image classification. For instance-level deletion, changing the forget set changes the corresponding retrained reference, requiring a new population of retrained models and an instance-conditioned or local diagnostic rather than the class-mean shift used in our analysis. For concept unlearning, the target concept may not correspond to a uniquely defined set of training examples, making a calibrated retraining reference infeasible until the concept is operationalized as a specific training set or distribution. We therefore leave the extension of the proposed protocol to instance-level and concept-level unlearning for future work.

As a whole, these considerations clarify the scope of our claim rather than weaken it. Output-level evaluation remains useful, retraining is not always operationally feasible, while generalization to foundation models, large-scale datasets, and instance-level unlearning remains future work. Our empirical evidence establishes robustness within class-level image unlearning across CIFAR-10, CIFAR-100, and TinyImageNet; ResNet-18/50 and ViT-Tiny; multiple forget classes; and multiple random seeds. The (K=10) calibration is established in the primary CIFAR-10/ResNet-18 setting, while the same calibration principle can be applied to other settings when multiple valid retrains and appropriate diagnostics are available, using a setting-specific empirical reference distribution. We therefore scope our structural conclusions to class-level image unlearning while providing a concrete calibration framework for extending the analysis to other settings. Our central point is that current machine unlearning is often evaluated using a signal that is too weak to detect the structured hidden failure mode studied in this paper.

\newpage
\onecolumn



\end{document}